\documentclass[11pt]{article}
\usepackage[letterpaper]{geometry}
\usepackage[parfill]{parskip}
\usepackage{amsmath,amsthm,amssymb,bbm}
\usepackage{mathtools}
\usepackage{cases}
\usepackage{dsfont}
\usepackage{microtype}

\usepackage{subfigure}
\usepackage{algorithm,algorithmic}
\usepackage{color}
\usepackage{appendix}

\usepackage{authblk}

\usepackage{url}
\usepackage[authoryear]{natbib}
\usepackage[colorlinks,citecolor=blue,urlcolor=blue,linkcolor=blue,linktocpage=true]{hyperref}
\usepackage{cleveref}
\crefformat{equation}{(#2#1#3)}
\crefrangeformat{equation}{(#3#1#4) to~(#5#2#6)}
\crefname{equation}{}{}
\Crefname{equation}{}{}

\crefname{definition}{\textbf{definition}}{definitions}
\Crefname{definition}{Definition}{Definitions}
\crefname{assumption}{\textbf{assumption}}{assumptions}
\Crefname{assumption}{Assumption}{Assumptions}

\definecolor{maroon}{RGB}{192,80,77}

\newtheorem{theorem}{Theorem}[section]
\newtheorem{lemma}[theorem]{Lemma}

\newtheorem{remark}[theorem]{Remark}
\newtheorem{assumption}[theorem]{Assumption}

\newcommand{\TMIS}{Tabular-MIS estimator}
\newcommand{\SMIS}{State-MIS estimator}

\newcommand{\argmax}{\mathop{\mathrm{argmax}}}

\def\E{\mathbb{E}}
\def\P{\mathbb{P}}
\def\Cov{\mathrm{Cov}}
\def\Var{\mathrm{Var}}

\def\diag{\mathrm{diag}}

\def\R{\mathbb{R}}
\def\cA{\mathcal{A}}

\begin{document}

\title{\textbf{Asymptotically Efficient Off-Policy Evaluation for Tabular Reinforcement Learning}}

\author[1,2]{Ming Yin}
\author[2]{Yu-Xiang Wang}
\affil[1]{Department of Statistics and Applied Probability, UC Santa Barbara}
\affil[2]{Department of Computer Science, UC Santa Barbara}

\maketitle

\begin{abstract}
% !TEX root = mis_tabular_aistats.tex

We consider the problem of off-policy evaluation for reinforcement learning, where the goal is to estimate the expected reward of a target policy $\pi$ using offline data collected by running a logging policy $\mu$.  Standard importance-sampling based approaches for this problem suffer from a variance that scales exponentially with time horizon $H$, which motivates a splurge of recent interest in alternatives that break the ``Curse of Horizon'' \citep{liu2018breaking,xie2019towards}. In particular, it was shown that a marginalized importance sampling (MIS) approach can be used to achieve an estimation error of order $O(H^3/ n)$ in mean square error (MSE) under an episodic Markov Decision Process model with finite states and potentially infinite actions.
The MSE bound however is still a factor of $H$ away from a Cramer-Rao lower bound of order  $\Omega(H^2/n)$. 
In this paper, we prove that with a simple modification to the MIS estimator, we can asymptotically attain the Cramer-Rao lower bound, provided that the action space is finite. We also provide a general method for constructing MIS estimators with high-probability error bounds.  
%The result complements existing doubly robust estimators that achieve asymptotic efficiency \citep{kallus2019double} and provides a simpler alternative with explicit finite sample guarantees.

%but only in the case when the action space is finite. 

%\red{When the action space is unbounded, we prove a stronger information-theoretic lower bound of $\Omega(H^3)$, which certifies that the original MIS estimator is indeed rate-optimal in the large action setting.} 
\end{abstract}

\newpage
\tableofcontents
\newpage

\section{Introduction}
% !TEX root = mis_tabular_aistats.tex

%\subsection{Preliminary}
\label{sec:preliminary}
\emph{Off-policy evaluation} (OPE), which predicts the performance of a policy with data only sampled by a logging/behavior policy \citep{sutton2018reinforcement}, plays a key role for using reinforcement learning (RL) algorithms responsibly in many real-world decision-making problems such as marketing, finance, robotics, and healthcare. Deploying a policy without having an accurate evaluate of its performance could be costly, illegal, and can even break down the machine learning system. There is a large body of literature that studied the off-policy evaluation problem in both theoretical and application-oriented aspects. From the theoretical perspective, OPE problem is extensively studied in contextual bandits \citep{li2011unbiased,dudik2011doubly,swaminathan2017off,wang2017optimal} and reinforcement learning (RL) \citep{li2015toward,jiang2016doubly,thomas2016data,farajtabar2018more,xie2019towards} and the results of OPE studies have been applied to real-world applications including marketing \citep{theocharous2015personalized,thomas2017predictive} and education \citep{mandel2014offline}.

\noindent\textbf{Problem setup.} In the reinforcement learning (RL) problem the agent interacts with an underlying unknown dynamics which is modeled as a Markov decision process (MDP). An MDP is defined by a tuple $M=(\mathcal{S},\mathcal{A},r,P,d_1,H)$, where $\mathcal{S}$ and $\mathcal{A}$ are the state and action spaces, $P_t : \mathcal{S}\times \mathcal{A}\times \mathcal{S} \rightarrow [0,1]$ is the transition kernel with $P_t (s^\prime |s, a)$ representing the probability of seeing state $s'$ after taking action $a$ at state $s$, $r_t : \mathcal{S} \times \mathcal{A} \rightarrow \R$ is the mean reward function with $r_t(s,a)$ being the average immediate goodness of $(s,a)$ at time $t$. Also, $d_1$ is denoted as the initial state distribution and $H$ is the time horizon. The subscript $t$ in $P_t$ means the transition dynamics are non-stationary and could be different at each time $t$. A (non-stationary) policy $\pi : \mathcal{S}  \rightarrow \P_\mathcal{A}^H$\footnote{Here $\P^H_{\mathcal{A}}=\P_{\mathcal{A}}\times\P_{\mathcal{A}}\times\P_{\mathcal{A}}\times...\times\P_{\mathcal{A}}$, where ``$\times$'' represents Cartesian product and the product is performed for $H$ times.} assigns each state $s_t \in \mathcal{S}$ a distribution over actions at each time $t$, \emph{i.e.} $\pi_t(\cdot|s_t)$ is a probability simplex with dimension $|\mathcal{S}|$. For brevity we suppress the subscript $t$ of $\pi_t$ and denote $\pi(a_t|s_t)$ the p.m.f of actions given state at time $t$.

Given a target policy of interest $\pi$, then the distribution of one H-step trajectory $\tau=(s_1,a_1,r_1,...,s_H,a_H,r_H,s_{H+1})$ is specified by $\pi:=(d_1,\pi)$\footnote{For brevity, $\forall \pi$ we use $\pi$ to denote the pair $(d_1,\pi)$. This can be understood as: $\forall \pi,\;d^\pi_1=d_1$. }   as follows: $s_1\sim d^\pi_1$, for $t=1,...,H$, $a_t\sim \pi_t(\cdot|s_t)$ and random reward $r_t$ has mean $r_t(s_t,a_t)$. Then the value function under policy $\pi$ is defined as:   
\[
v^\pi=\E_\pi\left[\sum_{t=1}^H r_t\right]. 
\]
The OPE problem aims at estimating $v^\pi$ while given that $n$ episodic data\footnote{To distinguish the data from different episodes, we use superscript to denote which episode they belong to throughout the rest of the paper.} $\mathcal{D}=\left\lbrace (s_t^{(i)},a_t^{(i)},r_t^{(i)})\right\rbrace_{i\in[n]}^{t\in[H]} $ are actually coming from a different logging policy $\mu$.

\noindent\textbf{Existing methods.}
The classical way to tackle the problem of OPE relies on incorporating importance sampling weights (IS), which corrects the mismatch in the distributions under the behavior policy and target policy. Specifically, define the $t$-step importance ratio as $\rho_t:=\pi_t(a_t|s_t)/\mu_t(a_t|s_t)$, then it uses the cumulative importance ratio $\rho_{1:t} :=\prod_{t'=1}^t \rho_{t'}$ to create IS based estimators:
\begin{align*}
\widehat{V}_{\text{IS}}:=\frac{1}{n}\sum_{i=1}^n \widehat{V}_{\text{IS}}^{(i)}, \quad&\widehat{V}_{\text{IS}}^{(i)}:=\rho_{1:H}^{(i)}\cdot \sum_{t=1}^Hr_t^{(i)};\\
\widehat{V}_{\text{step-IS}}:=\frac{1}{n}\sum_{i=1}^n \widehat{V}_{\text{step-IS}}^{(i)},\quad &\widehat{V}_{\text{step-IS}}^{(i)}:=\sum_{t=1}^H\rho_{1:t}^{(i)}r_t^{(i)},
\end{align*}
where $\rho_{1:t}^{(i)}=\prod_{t'=1}^t \pi_{t'}(a_{t'}^{(i)}|s_{t'}^{(i)})/\mu_{t'}(a_{t'}^{(i)}|s_{t'}^{(i)})$. There are different versions of IS estimators including weighted IS estimators and doubly robust estimators \citep{murphy2001marginal,hirano2003efficient,dudik2011doubly,jiang2016doubly}.

Even though IS-based off-policy evaluation methods possess a lot of advantages (\emph{e.g.} {unbiasedness}), the variance of the cumulative importance ratios $\rho_{1:t}$ may grow exponentially as the horizon goes long. Attempts to break the barriers of horizon have been tried using model-based approaches \citep{liu2018representation,gottesman2019combining}, which builds the whole MDP using either parametric or nonparametric models for estimating the value of target policy. \citep{liu2018breaking} considers breaking the curse of horizon of time-invariant MDPs by deploying importance sampling on the average visitation distribution of state-action pairs, \citep{hallak2017consistent} considers leveraging the stationary ratio of state-action pairs to replace the trajectory weights in an online fashion and \citep{gelada2019off} further applies the same idea in the deep reinforcement learning regime. Recently, \citep{kallus2019double,kallus2019efficiently} propose double reinforcement learning (DRL), which is based on doubly robust estimator with cross-fold estimation of $q$-functions and marginalized density ratios. It was shown that DRL is asymptotically efficient when both components are estimated at fourth-root rates, however no finite sample error bounds are given.

\noindent\textbf{Our goal.}
In this paper, our goal is to obtain the optimality of IS-based methods through marginalized importance sampling (MIS). As an earlier attempt, \citet{xie2019towards} constructs MIS estimator by aggregating all trajectories that share the same state transition patterns to directly estimate the state distribution shifts after the change of policies from the behavioral to the target. However, as pointed out by \citet{kallus2019double} and Remark~4 in \citet{xie2019towards}, the MSE upper bound of MIS estimator is asymptotically inefficient by a multiplicative factor of $H$. \citet{xie2019towards} conjectures that the lower bound is not achievable in their infinite action setting.
%essentially sub-optimal with assumption only on frequent state observability.
To bridge the gap and ultimately achieve the optimality, we consider the Tabular MDPs, where both the state space and action space are finite (\emph{i.e.} $S=|\mathcal{S}|<\infty,A=|\mathcal{A}|<\infty$) and each state-action pair can be visited frequently as long as the logging policy $\mu$ does sufficient exploration (which corresponds to Assumption~\ref{assume2}).  Under the Tabular MDP setting, we can show the MSE upper bound of MIS estimator matches the Cramer-Rao lower bound provided by \citet{jiang2016doubly} by incorporating frequent action observability. To distinguish the difference, throughout the rest of paper we call the modified MIS estimator Tabular-MIS (TMIS) and the MIS estimator in \citet{xie2019towards} State-MIS (SMIS).

%\red{Other than the asymptotic optimality, Tabular-MIS estimator has only $H^2$ dependence in MSE, which in turn will make CR lower bound smaller. see Thenreom?}

%\red{Talk about how CR lower bound might not be asymptotically achievable in some cases.}
\subsection{Summary of results.}
%\noindent\textbf{Summary of results.}
This work considers the problem of off-policy evaluation for a finite horizon, nonstationary, episodic MDP under tabular MDP setting. We propose and analyze \TMIS{}, which closes the gap between Cramer-Rao lower bound provided by \citet{jiang2016doubly} (on the variance of any unbiased estimator for a simplified setting of an nonstationary episodic MDP) and the MSE upper bound of \SMIS{} \citep{xie2019towards}. We also provide a high probability result %in the \emph{PAC-learnable} sense 
by introducing a data-splitting type \TMIS, which retains the asymptotic efficiency while having an exponential tail.  To the best of our knowledge, Split-TMIS is the first IS-based estimator in OPE that achieves asymptotic sample efficiency while having finite sample guarantees in high probability.

Moreover, the calculation of \TMIS{} and Split-TMIS does not explicitly incorporate the importance weights, which in turn implies our off-policy evaluation algorithm can be implemented without needing to know logging probabilities $\mu$. Such logging-policy-free feature makes our TMIS estimator more practical in the real-world applications.

 %\red{Moreover, we discuss the results of TMIS in finite-horizon MDP with stationary dynamics and infinite-horizon MDP with discounted rewards}.   

\noindent\textbf{Key proof ingradients.} We use a modified version of fictitious estimator of \citet{xie2019towards} as the bridge to connect our estimator $\widehat{v}^\pi_{\text{TMIS}}$ with $v^\pi$. Different from \citet{xie2019towards} who directly analyzes transition dynamic $\widehat{P}^\pi_{t+1}(s_{t+1}|s_t)$, we need to do a finer decomposition $\widehat{P}^\pi_{t+1}(s_{t+1}|s_t)=\sum_{a_t}\widehat{P}_{t+1}(s_{t+1}|s_t,a_t)\pi(a_{t}|s_t)$ and analyze $\widehat{P}_{t+1}(s_{t+1}|s_t,a_t)$. Also, Bellman equations are leveraged for expressing the variance of TMIS recursively. For deriving the high probability bound, we design the data-splitting TMIS which not only matches perfectly with the standard concentration inequalities but also maintains the MSE of the same order as TMIS for certain appropriately chosen batch data-splitting size.   

\subsection{Other related work}
Markov Decision Processes have a long history of associated research \citep{puterman1994markov,sutton1998reinforcement}, but many theoretical problems in the basic tabular setting remain an active area of research as of today.
%have only be fully addressed recently.  
We briefly review the other settings that this problem and connect them to our results.

\noindent\textbf{Regret bound and sample complexity in the online setting.}   The bulk of existing work focuses on online learning, where the agent interacts with the MDP with the interests of identifying the optimal policy or minimizing the regret against the optimal policy.  The optimal regret is obtained by \citep{azar2017minimax} using a model-based approach which translates into a sample complexity bound of $O(H^3SA/\epsilon^2)$, which matches the lower bound of $\Omega(H^3SA/\epsilon^2)$\citep{azar2013minimax}.   The method is however not ``uniform PAC'' where the state of the art sample complexity remains $O(H^4SA/\epsilon^2)$ \citep{dann2017unifying}. Model-free approaches that require a space constraint of $O(HSA)$ were studied by \citet{jin2018q} which implies a sample complexity bound of $O(H^4SA/\epsilon^2)$.

\noindent\textbf{Sample complexity with a generative model.} Another sequence of work assumes access to a generative model where one can sample from $s_{t+1}$ and $r_t$ given any $s_t,a_t$ in time $O(1)$ \citep{kearns1999finite}. \citet{sidford2018near} is the first that establishes the optimal sample complexity of $\tilde{\Theta}( H^3SA/\epsilon^2)$ under this setting (counting $H$ generative model calls as one episode). \citet{agarwal2019optimality} establishes similar results by estimating the parameters of the MDP model using maximum-likelihood estimation.

Our setting is different in two ways. First, we consider a fixed pair of logging and target policy $\mu$ and $\pi$, so our bounds can depend explicitly $\pi$ and $\mu$ instead of $S,A$. Second, we do not have either online access to the environment (to change policies) or a generative model. Our high-probability bound with a direct union bound argument, implies a sample complexity of $\tilde{O}( H^3S^2A/\epsilon^2)$ for identifying the optimal policy, which is suboptimal up to a factor of $S$, but notably has the optimal dependence in $H$. We remark that achieving the optimal dependence in the planning horizon $H$ (or the discounting factor $(1-\gamma)^{-1}$ in the infinite horizon case) is generally tricky (see, e.g., the COLT open problem \citep{jiang2018open} for more details). The current paper is among the few instances where we know how to obtain the optimal parameters.

Finally, we acknowledge that tabular RL is a basic abstraction that is relatively far away from real applications, which might have unobserved states, continuous state, non-zero Bellman error in the value function approximation.  We leave generalization of the techniques in this paper to these more practical settings as future work.

%Very briefly survey other threads in RL theory research : unobserved states, continuous state,  realizability and bellman error, low-switching cost.

\section{Method}

% !TEX root = mis_tabular_aistats.tex

\subsection{Problem description}

In addition to the non-stationary, finite horizon tabular MDP $M=(\mathcal{S},\mathcal{A},r,P,d_1,H)$ (where $S:=|\mathcal{S}|<\infty$ and $A:=|\mathcal{A}|<\infty$), non-stationary logging policy $\mu$ and target policy $\pi$ in Section~\ref{sec:preliminary}, we denote $d^\mu_t(s_t,a_t)$ and $d^\pi_t(s_t,a_t)$ the induced joint state-action distribution at time $t$ and the state distribution counterparts $d^\mu_t(s_t)$ and $d^\pi_t(s_t)$, satisfying $d^{\pi}_t(s_t,a_t)=d^{\pi}_t(s_t)\cdot \pi(a_t|s_t)$.\footnote{For $\mu$, $d^{\mu}_t(s_t,a_t)=d^{\mu}_t(s_t)\cdot \mu(a_t|s_t)$.} The initial distributions are identical $d_1^\mu=d^\pi_1=d_1$. Moreover, we use $P^\pi_{i,j}\in\mathbb{R}^{S\times S}, \;\forall j<i$ to represent the state transition probability from step $j$ to step $i$ under policy $\pi$, where $P^\pi_{t+1,t}(s^\prime|s)=\sum_a P_{t+1,t}(s^\prime|s,a)\pi_t(a|s)$.  The marginal state distribution vector ${d}_t^\pi(\cdot)$ satisfies
${d}_t^\pi  = {P}^{\pi}_{t,t-1} {d}_{t-1}^\pi$.

Historical data $\mathcal{D}=\left\lbrace (s_t^{(i)},a_t^{(i)},r_t^{(i)})\right\rbrace_{i\in[n]}^{t\in[H]} $ was obtained by logging policy $\mu$ and we can only use $\mathcal{D}$ to estimate the value of target policy $\pi$, \emph{i.e.} $v^\pi$. Suppose we only assume knowledge about $\pi(a|s)$ for all $(s,a)\in\mathcal{S}\times\mathcal{A}$ and \emph{do not observe} $r_t(s_t,a_t)$ for any actions other than the noisy immediate reward $r_t^{(i)}$ after observing $s_t^{(i)},a_t^{(i)}$. The goal is to find an estimator to minimize the mean-square error (MSE):
\[
\mathrm{MSE}(\pi,\mu,M)=\mathbb{E}_\mu[(\widehat{v}^\pi-v^\pi)^2],
\]

\begin{assumption}[Bounded rewards]\label{assume1}
	 $\forall \;t=1,...,H$ and $i=1,...,n$, $0\leq r^{(i)}_t\leq R_{\max}$.
\end{assumption}
The bounded reward assumption can be relaxed to :  $\exists R_{\max},\sigma<+\infty$ such that $0\leq \mathbb{E}[r_t|s_t,a_t,s_{t+1}]\leq R_{\max}$, $\mathrm{Var}[r_t|s_t,a_t,s_{t+1}]\leq \sigma^2$ (as in \citet{xie2019towards}), for achieving Cramer-Rao lower bound. However, the boundedness will become essential for applying concentrate inequalities in deriving high probability bounds.

\begin{assumption}[Sufficient exploration]\label{assume2}
	Logging policy $\mu$ obeys that $d_m:=\min_{t,s_t}d^\mu_t(s_t)>0$.
\end{assumption}
This second assumption can be relaxed to a weaker version: require $d^\mu_t(s_t)>0$ whenever $d^\pi_t(s_t)>0$, and the corresponding $d_m:=\min_{t,s_t}\{d^\mu_t(s_t):d^\mu_t(s_t)>0\}$ but without changing the proof. Therefore, for the illustration purpose we stick to the above assumption. This assumption is always required for the consistency of off-policy evaluation estimator. 

\begin{assumption}[Bounded weights]\label{assume3}
	$\tau_s:=\max_{t,s_t}\frac{d^\pi_t(s_t)}{d^\mu_t(s_t)}<+\infty$ and $\tau_a:=\max_{t,s_t,a_t}\frac{\pi(a_t|s_t)}{\mu(a_t|s_t)}<+\infty$.
\end{assumption}
Assumption~\ref{assume3} is also necessary for discrete state and actions, as otherwise the second moments of the importance weight would be unbounded and the MSE of estimators will become intractable . The bound on $\tau_s$ is natural since $\tau_s\leq \max_{t,s_t}\frac{1}{d^\mu_t(s_t)}=\frac{1}{\min_{t,s_t}d^\mu_t(s_t)}=\frac{1}{d_m}$ and it is finite by the Assumption~\ref{assume2}; similarly, $\tau_a<\infty$ is also automatically satisfied if $\min_{t,s_t,a_t}\mu(a_t|s_t)>0$.  Finally, as we will see in the results, explicit dependence on $\tau_s,\tau_a$ and $d_m$ only appear in the low-order terms of the error bound.

\subsection{Tabular-MIS estimator}

To overcome the barrier caused by cumulative importance weights in IS type estimators, marginalized importance sampling directly estimates
the marginalized state visitation distribution $\widehat{d}_t$ and defines the MIS estimator:

\begin{equation}\label{MIS_def}
\widehat{v}^\pi_{\text{MIS}} =\frac{1}{n}\sum_{i=1}^n\sum_{t=1}^H\frac{\widehat{d}^\pi_t(s_t^{(i)})}{\widehat{d}^\mu_t(s^{(i)}_t)}\widehat{r}^\pi_t(s^{(i)}).
\end{equation}
and $\widehat{d}^\mu_t(\cdot)$ is directly estimated using the empirical mean, \emph{i.e.} $\widehat{d}^\mu_t(s_t):=\frac{1}{n}\sum_i \mathbf{1}(s_t^{(i)}=s_t):=\frac{n_{s_t}}{n}$ whenever $n_{s_t}>0$ and $\widehat{d}^\pi_t(s_t)/\widehat{d}^\mu_t(s_t)=0$ when $n_{s_t}=0$. Then the MIS estimator \eqref{MIS_def} becomes:
\begin{equation}\label{MIS}
\widehat{v}^\pi_{\text{MIS}}=\sum_{t=1}^H\sum_{s_t}\widehat{d}^\pi_t(s_t)\widehat{r}^\pi_t(s_t)
\end{equation}

\paragraph{Construction of \SMIS .} Based on the estimated marginal state transition $\widehat{d}_t^\pi  = \widehat{P}^{\pi}_t \widehat{d}_{t-1}^\pi$, \SMIS in \citet{xie2019towards} directly estimates the state transition $P^\pi_t(s_t|s_{t-1})$ and state reward $r^\pi_t(s_t)$ as:

  \begin{align*}
  \widehat{P}^{\pi}_t(s_{t} | s_{t-1})  =&  \frac{1}{n_{s_{t-1}}} \sum_{i=1}^{n}   \frac{\pi( a_{t-1}^{(i)}| s_{t-1})}{\mu( a_{t-1}^{(i)}| s_{t-1})} %\\ & use this for aistats submission
  \cdot \mathbf{1}((s_{t-1}^{(i)},s_t^{(i)},a_t^{(i)}) = (s_{t-1},s_t,a_t) );
  \\ 
  \widehat{r}_t^{\pi}(s_t)  =&  \frac{1}{n_{s_t}}\sum_{i=1}^n \frac{\pi(a_t^{(i)}|s_t)}{\mu(a_t^{(i)}|s_t)} r_t^{(i)} \cdot \mathbf{1}(s_t^{(i)} = s_t).
  \label{eq:mis-reward}
  \end{align*}
  
\SMIS{} directly constructs state transitions $\widehat{P}^{\pi}_t(s_{t} | s_{t-1}) $ without explicitly modeling actions. Therefore, it is still valid when action space $\mathcal{A}$ is unbounded. However, importance weights must be explicitly utilized for compensating the discrepancy between $\mu$ and $\pi$ and the knowledge of $\mu(a|s)$ at each state-action pair $(s,a)$ is required.  

\paragraph{Construction of \TMIS .} 
  
Since tabular MDP setting assumes finite states and actions, we can go beyond importance weights and construct empirical estimates for $\widehat{P}_{t+1}(s_{t+1}|s_{t},a_{t})$ and $\widehat{r}_t(s_t,a_t)$ as: 
{\small
\begin{equation}\label{eq:tabular_MIS_construction}
%\hspace{-3.5mm}
\begin{aligned}
\widehat{P}_{t+1}(s_{t+1}|s_{t},a_{t})&=\frac{\sum_{i=1}^n\mathbf{1}[(s^{(i)}_{t+1},a^{(i)}_t,s^{(i)}_t)=(s_{t+1},s_{t},a_{t})]}{n_{s_{t},a_{t}}}\\
\widehat{r}_t(s_t,a_t)&=\frac{\sum_{i=1}^n r_t^{(i)}\mathbf{1}[(s^{(i)}_t,a^{(i)}_t)=(s_t,a_t)]}{n_{s_t,a_t}},\\
\end{aligned}
\end{equation}
}
where we set $\widehat{P}_{t+1}(s_{t+1}|s_{t},a_{t})=0$ and $\widehat{r}_t(s_t,a_t)=0$ if $n_{s_t,a_t}=0$, with $n_{s_t,a_t}$ the empirical visitation frequency to state-action $(s_t,a_t)$ at time $t$. The corresponding estimation of $\widehat{P}^\pi_t(s_t|s_{t-1})$ and $\widehat{r}^\pi_t(s_t)$ are defined as:

\begin{equation}\label{eq:tabular_MIS_construction2}
\begin{aligned}
\widehat{P}^\pi_t(s_t|&s_{t-1})=\sum_{a_{t-1}}\widehat{P}_t(s_t|s_{t-1},a_{t-1})\pi(a_{t-1}|s_{t-1}),\\
\widehat{r}^\pi_t(s_t)&=\sum_{a_t}\widehat{r}_t(s_t,a_t)\pi(a_t|s_t),\;\widehat{d}^\pi_t=\widehat{P}^\pi_t\widehat{d}^\pi_{t-1}. 
\end{aligned}
\end{equation}
In conclusion, by using the same estimator for $\widehat{d}^\mu_t$, $\widehat{v}^\pi_{\text{TMIS}}$ and $\widehat{v}^\pi_{\text{SMIS}}$ share the same form of \eqref{MIS}. However, \TMIS{} constructs a different estimation of component $\widehat{d}^\pi_t$ though \eqref{eq:tabular_MIS_construction}-\eqref{eq:tabular_MIS_construction2} by leveraging the fact that each state-action pair is visited frequently under tabular setting.

The motivation of MIS-type estimators comes from the fact that we have a nonstationary MDP model and its underlying state marginal transition follows ${d}^\pi_t={P}^\pi_t{d}^\pi_{t-1}$. The MIS estimators are then obtained by using corresponding plug-in estimators for each different components (\emph{i.e.} $\widehat{d}^\pi_t$ for ${d}^\pi_t$, $\widehat{P}^\pi_t$ for ${P}^\pi_t$). On the other hand, IS-type estimators design the value function in a more straightforward way without needing to estimate the transition environment \citep{mahmood2014weighted}. Therefore in this sense MIS-type estimators are essentially model-based estimators with the model of interactive environment $M=(\mathcal{S},\mathcal{A},r,P,d_1,H)$.

%\red{However, model-based approaches are not well-understood even in tabular setting.  Some discussions and summary of the understanding of tabular methods are here: \url{http://nanjiang.cs.illinois.edu/files/cs598/note3.pdf}
%In particular, check the "Simulation Lemma", and check out the various ``naive approaches'' to analyze the model-based methods.}

\section{Main Results}
% !TEX root = mis_tabular_aistats.tex

We now show that our \TMIS{} achieves the asymptotic Cramer-Rao lower bound for DAG-MDP \citep{jiang2016doubly} and therefore is asymptotically sample efficient. To formalize our statement, we pre-specify the following boundary conditions: $r_0(s_0) \equiv 0$, $\sigma_0(s_0,a_0)\equiv 0$,$\frac{d_0^\pi(s_0)}{d_0^\mu(s_0)}\equiv 1$, $\frac{\pi(a_0|s_0)}{\mu(a_0|s_0)} \equiv 1$, $V_{H+1}^\pi \equiv 0$, and, as a reminder, $\tau_a := \max_{t,s_t,a_t}\frac{\pi(a_t|s_t)}{\mu(a_t|s_t)}$ and $\tau_s := \max_{t,s_t} \frac{d_t^\pi(s_t)}{d_t^\mu(s_t)}$. 

\begin{theorem}\label{thm:main}	
	%		\leq&  \frac{4}{n}\sum_{h=0}^H \sum_{s_h}  \frac{ d_{h}^\pi(s_h)^2}{d_{h}^\mu(s_h)} \sum_{a_h}\frac{\pi(a_h|s_h)^2}{\mu(a_h|s_h)} (\| V_{h+1}^{\pi}\|_{T_{h+1}(\cdot|s_h,a_h)}^2 + \sigma_h^2(s_h,a_h) +  r_h(s_h,a_h)^2)\\
	
	Suppose the $n$ episodic historical data $\mathcal{D}=\left\lbrace (s_t^{(i)},a_t^{(i)},r_t^{(i)})\right\rbrace_{i=1,...,n}^{t=1,...,H} $ is obtained by running a logging policy $\mu$ and $\pi$ is the new target policy which we want to test. If the number of episodes $n$ satisfies
	{\small
	$$n > \max\left[ \frac{16\log n}{\min_{t,s_t,a_t}d_t^\mu(s_t,a_t)},\frac{4H \tau_a \tau_s}{\min_{t,s_t}\max\{d_{t}^\pi(s_{t}),d_{t}^\mu(s_{t})\} }  \right],$$ %\frac{4\log_{e/2}( n S)}{\min_{t,s_t}d_{t}^\mu(s_{t}) } $$ 
}
then under Assumption~\ref{assume1}-\ref{assume3} our \TMIS{} $\widehat{v}_{\mathrm{TMIS}}^\pi$ has the following Mean-Square-Error upper bound:
\begin{equation}\label{eq:main_result}
	\begin{aligned}
	%& for aistats
	&\E[ (\widehat{v}_{\mathrm{TMIS}}^\pi -  v^\pi)^2]\\
	\leq & \frac{1}{n}\sum_{h=0}^H \sum_{s_h,a_h}  \frac{ d_{h}^\pi(s_h)^2}{d_{h}^\mu(s_h)}\frac{\pi(a_h|s_h)^2}{\mu(a_h|s_h)}%\\ \cdot&  for aistats
	\Var\left[ (V_{h+1}^\pi(s_{h+1}^{(1)}) +  r_h^{(1)})\middle| s_{h}^{(1)}=s_h,a_h^{(1)}=a_h\right]% \\&  for aistats
	\cdot \left(1+\sqrt{\frac{16\log n}{n\min_{t,s_t}d_t^\mu(s_t)}}\right)\\
	+&O(\frac{\tau_a^2\tau_sH^3}{n^2\cdot d_m}),%\frac{3}{n^2}H^3SAR_{\max}^2\\
	%+&\frac{8\tau_a\tau_s}{n^2}\sum_{h=1}^H\sum_{s_h,a_h} \frac{h\cdot d_{h}^\pi(s_h)}{d_{h}^\mu(s_h)}\frac{\pi(a_h|s_h)^2}{\mu(a_h|s_h)}\\
	%\cdot&\Var\left[ (V_{h+1}^\pi(s_{h+1}^{(1)}) +  r_h^{(1)})\middle| s_{h}^{(1)}=s_h,a_h^{(1)}=a_h\right], 	
	\end{aligned}
\end{equation}
	where the value function under $\pi$ is defined as:
	$V_h^\pi(s_h) :=   \E_\pi\left[ \sum_{t=h}^{H}r_t^{(1)} \middle| s_{h}^{(1)}=s_h\right],\;  \forall h\in\{1,2,...,H\}.
	$
\end{theorem}

The proof of this theorem, and all the other technical results we present in this section, are deferred to the appendix due to the space constraint. We summarize the novel ingredients in the proof in Section~\ref{sec:proof_sketch}. Before that, we make a few remarks about a few interesting aspects of this result.

\vspace{0.1in}
\begin{remark}[Asymptotic efficiency and local minimaxity]\label{remark:asym}
	The error bound implies that 
	\newline
	$\lim_{n\rightarrow \infty}{n} \cdot\E[ (\widehat{v}_{\mathrm{TMIS}}^\pi -  v^\pi)^2]$   %the MSE of \TMIS{} obeys that
		{\small 
	\begin{align*}
	\sum_{t=0}^{H}\E_\mu\left[ \frac{d^\pi(s_t^{(1)},a_t^{(1)})^2}{d^\mu(s_t^{(1)},a_t^{(1)})^2}\Var\Big[V_{t+1}^\pi(s_{t+1}^{(1)})+r_t^{(1)}\Big| s_{t}^{(1)}, a_t^{(1)}\Big]\right].
	\end{align*}
}
	This exactly matches the CR-lower bound in \citet[Proposition 3]{jiang2016doubly} for DAG-MDP\footnote{\citet{jiang2016doubly} focused on the special case with deterministic reward only at $t=H$. It is straightforward to show that the above expression is the CR-lower bound in the general tabular setting.}.
	In contrast, the \SMIS{} in \citep{xie2019towards} achieves an asymptotic MSE of
	{\small 
	\begin{equation}\label{MSE:SMIS}
	\sum_{t=0}^{H}\E_\mu\left[ \frac{d^\pi(s_t^{(1)})^2}{d^\mu(s_t^{(1)})^2}\Var\Big[\frac{\pi(a_t^{(1)} | s_t^{(1)})}{\mu( a_t^{(1)}| s_t^{(1)})}(V_{t+1}^\pi(s_{t+1}^{(1)})+r_t^{(1)})\Big| s_{t}^{(1)}\Big]\right].
	\end{equation}
}
\end{remark}
We note that while in classical literature CR-lower bound is often used to lower bound the variance of \emph{unbiased} estimators, the modern theory of estimation establishes that it is also the correct asymptotic minimax lower bound for the MSE of \emph{all} estimators in every local neighborhood of the parameter space \citep[see, e.g.,][Chapter 8]{van2000asymptotic}. In other words, our results imply that  \TMIS{} is asymptotically, locally, uniformly minimax optimal, namely, optimal for every problem instance separately. 

It is worth pointing out that while asymptotically efficient estimators for this problem in related settings have been proposed in independent recent work \citep{kallus2019double,kallus2019efficiently}, our estimator is the first that comes with finite sample guarantees with an explicit expression on the low-order terms.  Moreover, our estimator demonstrates that doubly robust estimation techniques is not essential for achieving asymptotic efficiency.

\begin{remark}[Simplified finite sample error bound]
	The theory implies that there is universal constants $C_1,C_2$ such that for all $n \geq  C_1H\frac{\tau_a}{d_m}$, i.e., when we have a just visited every state-action pair for $\Omega(H)$ times,
	$
	\E[ (\widehat{v}_{\mathrm{TMIS}}^\pi -  v^\pi)^2] = C_2H^2\tau_a\tau_sR_{\max}^2 / n.
	$
\end{remark}

In deriving the above remark, we used the somewhat surprising observation that $$\sum_{t=1}^{H}\E_\pi\left[ \Var\Big[V_{t+1}^\pi(s_{t+1}^{(1)})+r_t^{(1)}\Big| s_{t}^{(1)}, a_t^{(1)}\Big]\right] \leq H^2R_{\max }^2.$$ Note that we are summing $H$ quantities that are potentially on the order of $H^2R_{\max }^2$, yet no additional factors of $H$ shows up.  This observation is folklore and has been used in deriving tight results for tabular RL in \citep[e.g.,][]{azar2017minimax}. It can be proven using the following decomposition of the variance of the empirical mean estimator and the fact that it  is bounded by $H^2R_{\max}^2/4$.

	\begin{lemma}\label{lem:H3_to_H2}
%=======
%	In deriving the above remark, we used the somewhat surprising observation that $$\sum_{t=0}^{H}\E_\pi\left[ \Var\Big[V_{t+1}^\pi(s_{t+1}^{(1)})+r_t^{(1)}\Big| s_{t}^{(1)}, a_t^{(1)}\Big]\right] \leq H^2R_{\max }^2.$$ Note that we are summing $H$ quantities that are potentially on the order of $H^2R_{\max }^2$, yet no additional factor of $H$ shows up.  This observation is folklore and has been used in deriving tight results for tabular RL in \citep[e.g.,][]{azar2017minimax}. It can be proven by the following decomposition of the variance and the fact that it is also bounded by $H^2R_{\max}^2/4$.
%	\begin{lemma}
%>>>>>>> e0483cd7bb54003b76de9fc6b65499e713d1581a
For any policy $\pi$ and any MDP.
{\small
\begin{align*}
&\mathrm{Var}_\pi\left[\sum_{t=1}^H r^{(1)}_t\right] = \sum_{t=1}^H \Big(\E_\pi\left[ \mathrm{Var}\left[r^{(1)}_t+V^\pi_{t+1}(s_{t+1}^{(1)}) \middle|s^{(1)}_t,a^{(1)}_t\right] \right]\\
&\quad +  \E_\pi\left[ \mathrm{Var}\left[  \E[r^{(1)}_t+V^\pi_{t+1}(s_{t+1}^{(1)}) | s^{(1)}_t, a^{(1)}_t]  \middle|s^{(1)}_t\right] \right]\Big).
\end{align*}
}
	\end{lemma}
The proof, which applies the law-of-total-variance recursively, is deferred to the appendix.
	\begin{remark}[When $\pi = \mu$]
		One surprising observation is that \TMIS{} improves the efficiency even for the on-policy evaluation problem when $\pi=\mu$.  In other word,  the natural Monte Carlo estimator of the reward in the on-policy evaluation problem is in fact asymptotically inefficient.
	\end{remark}
%	See further discussion in Section~\ref{sec:discussion_H}.

\subsection{Building blocks of the analysis}\label{sec:proof_sketch}
At the high level, the techniques we used, including the idea of fictitious estimator and peeling the variance (expectation) of fictitious estimator $\widetilde{v}^\pi$ from behind by applying total law of variances (expectations) repeatedly, are consistent with \citet{xie2019towards}. 

In addition to the above techniques, we leverage the fact of frequent state-action visitations in our design of TMIS estimator and based on that we are able to achieve an asymptotic lower Mean Square Error (MSE) bound. The main components are the followings.

\noindent\textbf{Fictitious \TMIS.} Fictitious \TMIS{} $\widetilde{v}^\pi_\text{TMIS}$ is a modified version of $\widehat{v}^\pi_\text{TMIS}$ with $\widehat{P}^\pi_{t+1}(\cdot|s_t,a_t)$, $\widehat{r}^\pi_t(s_t,a_t)$ replaced by the underlying true ${P}^\pi_{t+1}(\cdot|s_t,a_t)$, ${r}^\pi_t(s_t,a_t)$ when the visitation frequency of state-action pair $(s_t,a_t)$ is insufficient (\emph{e.g.} $n_{s_t,a_t}<O(nd^\mu_t(s_t,a_t))$). In other words, fictitious \TMIS{} $\widetilde{v}^\pi_\text{TMIS}$ remains every part of $\widehat{v}^\pi_\text{TMIS}$ unchanged except the following:
\begin{equation}\label{eq:fictitious_r}
 \widetilde{r}_t(s_t,a_t) =  \begin{cases}
\widehat{r}_t(s_t,a_t)  & \mbox{ if } n_{s_t,a_t} \geq n d_t^\mu(s_t,a_t)(1-\theta)\\
r_t(s_t,a_t) &\mbox{ otherwise;}
\end{cases}
\end{equation}
and
\begin{equation}\label{eq:fictitious_p}
\widetilde{P}_{t+1,t}(\cdot | s_{t},a_{t})  = \begin{cases}
\widehat{P}_{t+1,t}    &\mbox{ if }n_{s_{t},a_t} \geq n d_t^\mu(s_{t},a_t)(1-\theta)\\
P_{t+1,t} &\mbox{ otherwise,}
\end{cases}
\end{equation}
where $\theta$ is the parameter constrained by  $0<\theta<1$, which we will choose later in the proof.

This slight modification makes $\widetilde{v}^\pi_\text{TMIS}$ no longer  implementable using the logging data $\mathcal{D}$, but it does provide an unbiased estimator of $v^\pi$ (Lemma~\ref{lem:unbiasedness_fictitious} in appendix) and, most importantly, it is easier to do theoretical analysis on $\widetilde{v}^\pi_\text{TMIS}$ than on $\widehat{v}^\pi_\text{TMIS}$. Moreover, Multiplicative Chernoff bound (Lemma~\ref{lem:chernoff_multiplicative} in appendix) helps to find the connection between $\widetilde{v}^\pi_\text{TMIS}$ and $\widehat{v}^\pi_\text{TMIS}$. 

\noindent\textbf{Peeling arguments using the total law of variance (expectation).} The core idea in analyzing the variance of $\widetilde{v}^\pi$ is to peel the variance from behind (start from time $H$ to $1$) and the peeling tool we used here is through marriaging the standard Bellman equations with the total law of variance. Lemma~\ref{lem:inter_recursion} (in appendix) shows this spirit and it is used repeatedly throughout the whole analysis. Beyond that, the peeling argument can be used to prove the dependence in $H$ is only $H^2$ for our \TMIS. This result explicates that $H^2$ is enough for TMIS to evaluate a particular policy and this is different from SMIS, which in general requires the dependence of $H^3$ for off-policy evaluation.

\subsection{A high-probability bound with data-splitting TMIS.}

% Maybe talk about the data splitting approach  +  Bernstein's inequality and leave the Martingale concentration inequality approach to our next paper.

\TMIS{} provides the asymptotic optimal variance bound of order $O(H^2SA/n)$ and based on that it is natural to ask the related learning question: whether TMIS can further achieve a high probability bound with the same sample complexity? We figure out that the standard concentration inequalities (\emph{e.g.} Hoeffding's inequality, Bernstein inequality) cannot be directly applied because of the highly correlated structures of the \TMIS. To address this problem we design the following data split version of TMIS and as we will see, the original TMIS is essentially a special case of data-splitting TMIS. 

\noindent\textbf{Data splitting \TMIS.} Assume the total number of episodes $n$ can be factorized as $n=M\cdot N$, where $M,N>1$ are two integers,\footnote{In general this might not be true, \emph{e.g.} if $n$ is prime number. However, we can resolve it by choosing $M=\lfloor n/N\rfloor$. } and we can partition the data $\mathcal{D}$ into $N$ folds with each fold $\mathcal{D}^{(i)}$ ($i=1,...,N$) has $M$ different episodes, or in other words, we split the $n$ episodes evenly. Then by the i.i.d. nature of $n$ episodes, we have $\mathcal{D}^{(1)},\mathcal{D}^{(2)},...,\mathcal{D}^{(N)}$ are independent collections.

 For each $\mathcal{D}^{(i)}$, we can create a \TMIS{} $\widehat{v}^{\pi(i)}_{\text{TMIS}}$
(for notation simplicity we use $\widehat{v}^\pi_{(i)}$ to denote $\widehat{v}^{\pi(i)}_{\text{TMIS}}$ in the future discussions) using its own $M$ episodes. Then $\widehat{v}^\pi_{(1)},\widehat{v}^\pi_{(2)},...,\widehat{v}^\pi_{(N)}$ are independent of each other and we can use the empirical mean to define the data splitting \TMIS and the corresponding fictitious version:
\begin{equation}\label{eq:data_split_TMIS}
\widehat{v}^\pi_{\text{split}}=\frac{1}{N}\sum_{i=1}^N \widehat{v}^\pi_{(i)},
\qquad
\widetilde{v}^\pi_{\text{split}}=\frac{1}{N}\sum_{i=1}^N \widetilde{v}^\pi_{(i)},
\end{equation}
where each $\widetilde{v}^\pi_{(i)}$ is the fictitious estimator of $\widehat{v}^\pi_{(i)}$.

The data splitting TMIS estimator explicitly characterizes the independence of $n$ different episodes by grouping them into $N$ chunks. Chunks are independent of each other and taking the average over all $\widehat{v}^\pi_{(i)}$ ($i=1,...,N$) will guarantee the validity of using concentration inequalities. 

More importantly, the data splitting TMIS estimator holds the same information-theoretical variance lower bound as the non-data splitting TMIS estimator, which is not surprising since the non-data splitting TMIS estimator is just the special case of the data splitting \TMIS{} with $N=1$. This idea is summarized into the following theorem:

\begin{theorem}\label{thm:var_bound_data_splitting_TMIS }Using $n$ i.i.d. episodic data from a near-uniform\footnote{Near-uniform here means: $\min_{t,s_t,a_t}d_t^\mu(s_t,a_t)>\Omega(1/(SA))$.} logging policy $\mu$ and suppose $M$, the number of episodes for each $\mathcal{D}^{(i)}$, satisfies:
	$M > \max\left[ O(SA\cdot\mathrm{Polylog} (S,H,A,n)),O(H \tau_a \tau_s)  \right],$
	  then the data splitting \TMIS{} obeys: 
\[%\begin{equation}\label{eq:data_splitting_lower_bound}
\E[ (\widehat{v}_{\mathrm{split}}^\pi -  v^\pi)^2]\leq O(\frac{H^2SA}{n}).
\]%\end{equation}	
\end{theorem}

\begin{remark}
	The condition in Theorem~\ref{thm:var_bound_data_splitting_TMIS } is achieveable. For example, choose $M\approx \sqrt{n}$, then the condition holds when $n$ is sufficiently large. 
\end{remark}

 \noindent\textbf{High probability bound. } By coupling the data splitting techniques with the boundedness of \TMIS{} (\emph{i.e.} $\widehat{v}^\pi\leq HR_{\max}, \widetilde{v}^\pi\leq HR_{\max}$, see Lemma~\ref{lem:bounded_estimator} in appendix), we can apply concentration inequalities to show the difference between $\widehat{v}^\pi_{\mathrm{split}}$ and $v^\pi$ is bounded by order $\widetilde{O}(\sqrt{H^2SA/n})$, which is summarized into the following theorem.
 
 \begin{theorem}\label{thm:high_probability}
 	Suppose $n$ i.i.d. episodic historical data comes from a near-uniform logging policy $\mu$ and suppose $M$, the number of episodes in each $\mathcal{D}^{(i)}$, satisfies:
 	$\widetilde{O}(n\cdot SA)\geq M$ and $M> \max\left[ O(SA\cdot \mathrm{Polylog} (S,H,A,n,1/\delta)),O(H \tau_a \tau_s)  \right]$.
 	Then we have with probability $1-\delta$, the data splitting \TMIS{} obeys: 
 % \red{	{\small
 %  	\begin{align*}
 %  	&|\widehat{v}_{\mathrm{split}}^\pi -  v^\pi|\leq
 %  	\sqrt{\frac{2V\log(2/\delta)}{n}}+\widetilde{O}\big( \sqrt{\frac{VSA}{n^{1.5}}}+ \frac{(SAH)^{1.5}}{n}\big),\\
 %  	&\quad \mathrm{with}\;V :=\\
 %  	&\sum_{t=0}^{H}\E_\mu\left[ \frac{d^\pi(s_t^{(1)},a_t^{(1)})^2}{d^\mu(s_t^{(1)},a_t^{(1)})^2}\Var\Big[V_{t+1}^\pi(s_{t+1}^{(1)})+r_t^{(1)}\Big| s_{t}^{(1)}, a_t^{(1)}\Big]\right].
 %  	\end{align*}
 %  }
 % % 	$\sum_{t=0}^{H}\E_\mu\left[ \frac{d^\pi(s_t^{(1)},a_t^{(1)})^2}{d^\mu(s_t^{(1)},a_t^{(1)})^2}\Var\Big[V_{t+1}^\pi(s_{t+1}^{(1)})+r_t^{(1)}\Big| s_{t}^{(1)}, a_t^{(1)}\Big]\right].$
 %  	Here $\widetilde{O}$ absorbs all $\mathrm{Polylog}(N,S,H,A,\log(1/\delta))$ terms.}
 
 	\[
 |\widehat{v}^\pi_{\mathrm{split}}-v^\pi|\leq \widetilde{O}(\sqrt{\frac{H^2SA}{n}}).
 \]

 \end{theorem}

The proof Theorem~\ref{thm:high_probability} relies on bounding the difference between $\widehat{v}_{\mathrm{split}}^\pi$ and $\widetilde{v}_{\mathrm{split}}^\pi$ using Multiplicative Chernoff bound and bounding the difference between $\widetilde{v}_{\mathrm{split}}^\pi$ and ${v}^\pi$ using Bernstein inequality. During the process of bounding $|\widehat{v}_{\mathrm{split}}^\pi-\widetilde{v}_{\mathrm{split}}^\pi|$ we observe that a stronger uniform bound can be derived. In fact, this bound is $0$. We formalize it into the following lemma. 

\begin{lemma}\label{lem:uniform_bound}
	Suppose $n$ i.i.d. episodic historical data comes from a near-uniform logging policy $\mu$ and suppose $M$, the number of episodes in each $\mathcal{D}^{(i)}$, satisfies:
	$M > \max\left[ O(SA\cdot \mathrm{Polylog} (S,H,A,N,1/\delta)),O(H \tau_a \tau_s)  \right]$.Then we have with probability $1-\delta$,
	\[
	\sup_{\pi\in\prod}|\widehat{v}_{\mathrm{split}}^\pi-\widetilde{v}_{\mathrm{split}}^\pi|=0
	%\leq O\big(\sqrt{\frac{SA\log(NHSA/\delta)}{n}}\big),
	\]
	Since $n=N\cdot M$, therefore let $N=1$, $M=n$, we have: if $M > \max\left[ O(SA\cdot \mathrm{Polylog} (S,H,A,1/\delta)),O(H \tau_a \tau_s)  \right]$, then we have with probability $1-\delta$,
	\[
	\sup_{\pi\in\prod}|\widehat{v}_{\mathrm{TMIS}}^\pi-\widetilde{v}_{\mathrm{TMIS}}^\pi|=0
	%\leq O\big(\sqrt{\frac{SA\log(HSA/\delta)}{n}}\big),
	\]
	where $\prod$ consists of all the $H$-step nonstationary policies.
\end{lemma}

\begin{remark}
The uniform difference bound between $\widehat{v}_{\mathrm{TMIS}}^\pi$ and $\widetilde{v}_{\mathrm{TMIS}}^\pi$ is obtained by observing the construction of fictitious estimator \eqref{eq:fictitious_r} and \eqref{eq:fictitious_p} are independent of the specific target policy $\pi$. This result tells the $	\sup_{\pi\in\prod}|\widehat{v}_{\mathrm{TMIS}}^\pi-\widetilde{v}_{\mathrm{TMIS}}^\pi|$ can be arbitrarily small with high probability and therefore does not depend on $H$ factor. This fact will help us to derive the correct dependence in $H$ for uniform convergence problem, see Section~\ref{sec:discussion}.
\end{remark}

\begin{algorithm}[H]
	\caption{Tabular MIS Off-Policy Evaluation}
	\label{alg:mainalgo}
	{\bfseries Input:} Logging data $\mathcal D = \{\{s_t^{(i)},a_t^{(i)},r_t^{(i)}\}_{t = 1}^{H }\}_{i = 1}^{n}$ from the behavior policy $\mu$. A target policy $\pi$ which we want to evaluate its cumulative reward.
	\begin{algorithmic}[1]
		\STATE Calculate the on-policy estimation of initial distribution $d_1(\cdot)$ by
		$
		\widehat d_1(s) := \frac{1}{n}\sum_{i = 1}^{n} \mathbf{1}(s_1^{(i)} = s),
		$
		and set $\widehat d_1^{\mu}(\cdot):=\widehat d_1(\cdot)$, $\widehat d_1^\pi(s):=\widehat d_1(\cdot)$.
		\FOR{$t = 2,3,\dotsc,H$}
		\STATE Choose all transition data at time step $t$, $\{s_t^{(i)},a_t^{(i)},r_t^{(i)}\}_{i = 1}^{n}$.
		\STATE Calculate the on-policy estimation of $d_{t}^{\mu}(\cdot)$ by
		%\begin{align}
		$
		\widehat d_{t}^{\mu}(s) := \frac{1}{n}\sum_{i = 1}^{n} \mathbf{1}(s_{t }^{(i)} = s).
		$
		%\end{align}
		\STATE Set the off-policy estimation of $\widehat{P}_{t}(s_{t}|s_{t-1},a_{t-1})$:
		\begin{equation*}
		\hspace{-6.5mm}
		\begin{aligned}
		&\widehat{P}_{t}(s_{t}|s_{t-1},a_{t-1})\\
		:=&\frac{\sum_{i=1}^n\mathbf{1}[(s^{(i)}_{t},a^{(i)}_{t-1},s^{(i)}_{t-1})=(s_{t},s_{t-1},a_{t-1})]}{n_{s_{t-1},a_{t-1}}}
		\end{aligned}
		\end{equation*}
		when $n_{s_{t-1},a_{t-1}}>0$. Otherwise set it to be zero.
		\STATE Estimate the reward function
		\begin{align*}
		\widehat{r}_t(s_t,a_t) := \frac{\sum_{i=1}^n r_t^{(i)}\mathbf{1}(s_t^{(i)}=s_t, a_t^{(i)} =a_t)}{\sum_{i=1}^n \mathbf{1}(s_t^{(i)}=s_t, a_t^{(i)} =a_t) }.
		\end{align*}
		when $n_{s_t,a_t}>0$. Otherwise set it to be zero.
		\STATE Set $\widehat{d}^\pi_t(\cdot)$ according to $\widehat{d}^\pi_t=\widehat{P}^\pi_t\widehat{d}^\pi_{t-1}$, with $\widehat{P}^\pi_t$ defined according to \eqref{eq:tabular_MIS_construction2}. Also, set $\widehat{r}^\pi_t(\cdot)$ according to \eqref{eq:tabular_MIS_construction2}.
		\ENDFOR
		\STATE Substitute the all estimated values above into \eqref{MIS_def} to obtain $\widehat v^\pi$, the estimated value of $\pi$.
	\end{algorithmic}
\end{algorithm}

\subsection{Some interpretations.}\label{subsec:implication}% in other aspects.}
%\subsection{Logging policy free algorithm.}
\noindent\textbf{Logging policy free algorithm.} 
We point out the implementation of \TMIS{} does not require the knowledge of logging policy $\mu$, as shown in Algorithm~\ref{alg:mainalgo},\ref{alg:secondalgo}.%\footnote{Algorithm~\ref{alg:secondalgo} is deferred to appendix due to space constraint.} only show this footnote in aistats version
This is critical in the sense that in the real-world sequential decision making problems, it is very likely the complete information about logging policy is not provided. This may happen due to mis-records or the lack of maintainance. By only using the historical data, tabular MIS off-policy evaluation is able to achieve the asymptotic efficiency. In contrast, the state MIS estimator always requires the full information about the logging policy.

\noindent\textbf{Connection to approximate MDP estimation.} 
Our TMIS is essentially an approximate MDP estimator (with the non-stationary dynamic transitions $P_t$ estimated by \textit{maximum likelihood estimator} (MLE))  except that we marginalize out the action in both $\widehat{r}^\pi_t(s)$ and $\widehat{d}_t^\pi(s)$ and provide an importance sampling interpretation. To the best of our knowledge, existing analysis of the fully model-based approach does not provide tight bounds. We give two examples. The seminal “simulation lemma” in \citet{kearns2002near} together with a naive concentration-type analysis gives only an $\widetilde{O}(\sqrt{H^4S^3A/n})$ bound in our setting. In a very recent compilation of improvements over this bound \citep{jiang2018notes}, this bound can be improved to either $\widetilde{O}(\sqrt{H^4S^2A/n})$ or $\widetilde{O}(\sqrt{H^6SA/n})$. Our result is the first that achieves the optimal $\widetilde{O}(\sqrt{H^2SA/n})$ rate regardless of whether it is the model-based or model-free approach.

% \noindent\textbf{Implication in stationary MDP setting.}
% Our \TMIS{} and its theoretical analysis can be extended to other MDP models. 

% First, finite-horizon MDP with stationary dynamics is a special case of our nonstationary setting. For tabular MDP with stationary transitions, each state and action are visited more frequently since there is no time $t$ to distinguish the transition dynamic $P(s^\prime|s,a)$. Therefore our estimator can be modified with $\widehat{P}(s^\prime|s,a)=\sum_{i,t}\mathbf{1}(s^{(i)}_{t+1}=s^\prime,s^{(i)}_t=s,a^{(i)}_t=a)/n_{s,a}$ and the fictitious counterpart $\widetilde{P}(s^\prime|s,a)=\widehat{P}(s^\prime|s,a)$ if $n_{s,a}>\Omega(n\sum_{t=1}^Hd^\mu_t(s,a))$ (otherwise $\widetilde{P}(s^\prime|s,a)={P}(s^\prime|s,a)$). In the case where $d^\mu_t(s,a)>\Omega({1/SA})$, then $1/n_{s,a}<O(SA/(nH))$ can help save an $H$ factor. We elaborate this implication in Appendix~\ref{sec:implication}.

\section{Experiments}\label{sec:experiment} 
% !TEX root = mis_tabular_aistats.tex

In this section, we present some empirical studies to demonstrate that our main theoretical results about \TMIS{} in Theorem~\ref{thm:main} are empirically verified. 

\noindent\textbf{Time-varying, non-mixing Tabular MDP}. We test our approach in simulated MDP environment where both the states and the actions are binary. Concretely, there are two states $s_0$ and $s_1$ and two actions $a_1$ and $a_2$. State $s_0$ always has probability $1$ going back to itself, regardless of the actions, \emph{i.e.} $P_t(s_0|s_0,a_1)=1$ and $P_t(s_0|s_0,a_2)=1$. For state $s_1$, at each time step there is one action (we call it $a$) that has probability $2/H$ going to $s_0$ and the other action (we call it $a'$) has probability $1$ going back to $s_1$, \emph{i.e.} $P_t(s_0|s_1,a)=2/H=1-P_t(s_1|s_1,a)$ and $P_t(s_1|s_1,a')=1$. Moreover, which action will make state $s_1$ go to state $s_0$ with probability $2/H$ is decided by a random parameter $p_t\in[0,1]$. If $p_t<0.5$, $a=a_1$ and if $p_t\geq 0.5$, $a=a_2$. One can receive reward $1$ at each time step if $t>H/2$ and is in state $s_0$, and will receive reward $0$ otherwise. Lastly, for state $s_0$, we set $\mu(\cdot|s_0)=\pi(\cdot|s_0)$; for state $s_1$, we set $\mu(a_1|s_1)=\mu(a_2|s_1)=1/2$ and $\pi(a_1|s_1)=1/4=1-\pi(a_2|s_1)$. 

%%%% change width to 40mm for aistats submission!!!!
\begin{figure}
	\centering     %%% not \center
	\subfigure[Different Episode $n$]{\label{fig:different_n}\includegraphics[width=70mm]{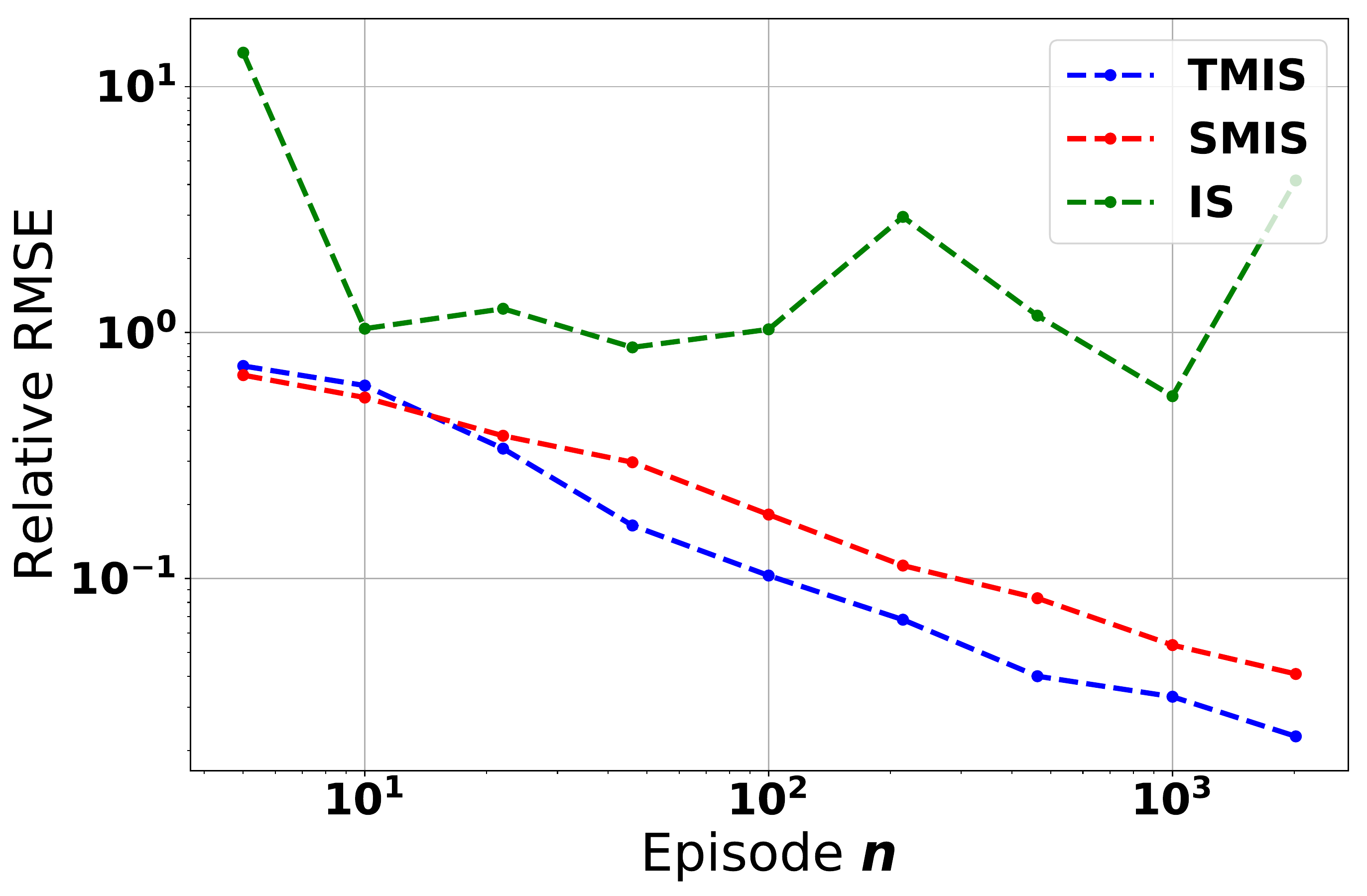}}
	\subfigure[Different Horizon $H$]{\label{fig:different_H}\includegraphics[width=70mm]{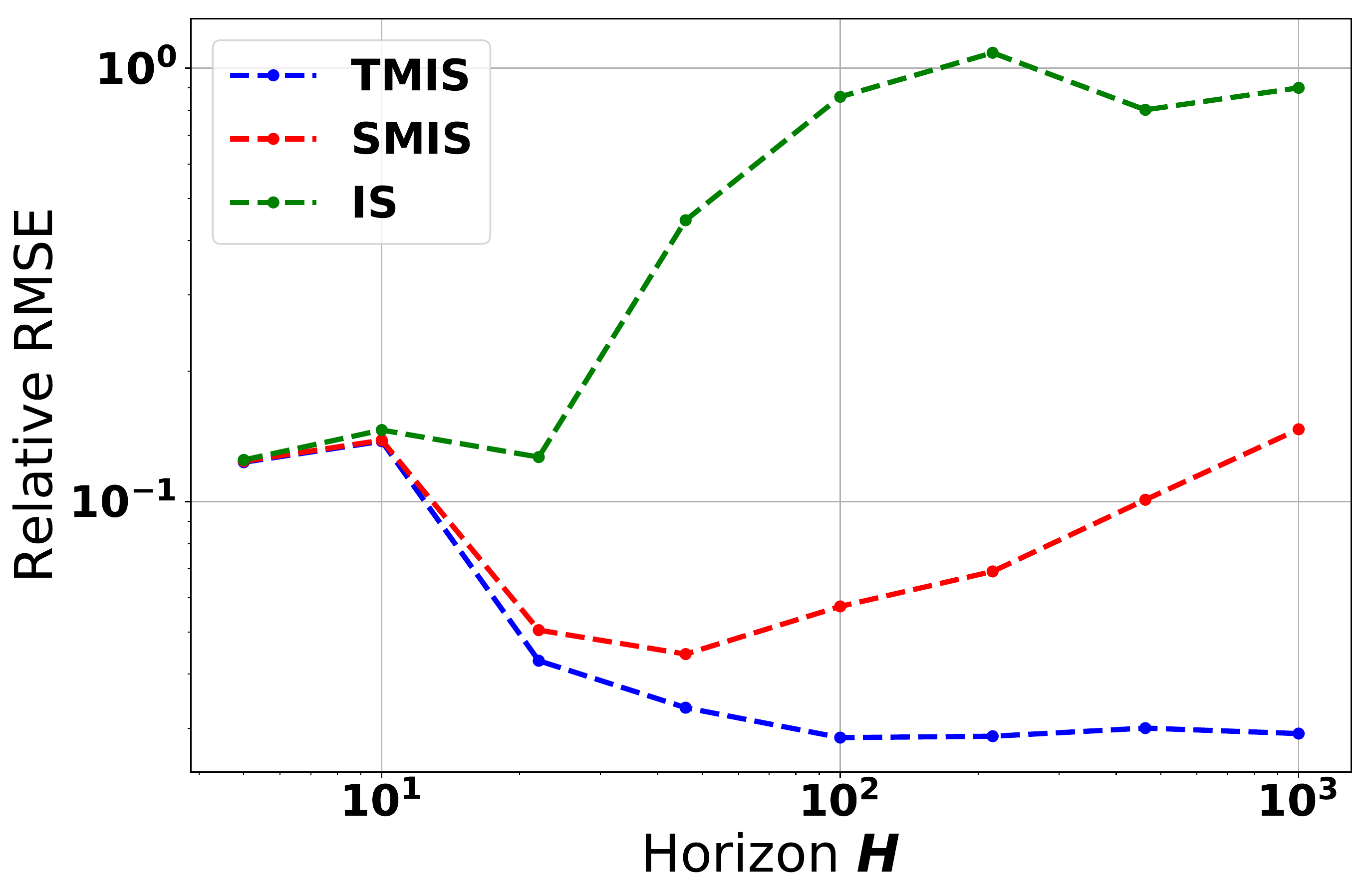}}
	\caption{Relative RMSE ($\sqrt{\text{MSE}}/v^\pi$) on Non-stationary Non-mixing MDP}
	\label{fig:main}
\end{figure}

Figure~\ref{fig:different_n} shows the asymptotic convergence
rates of relative RMSE with respect to the number of
episodes, given fixed horizon $H = 100$. Both SMIS and TMIS has a $O(1/\sqrt{n})$ convergence rate. The saving of $\sqrt{H}$ of TMIS over SMIS in this log-log plot is reflected in the intercept. Figure~\ref{fig:different_H} has fixed $n=1024$ with varying horizon $H$. Note since $v^\pi\approx O(H)$, therefore for TMIS our theoretical result implies $\sqrt{\text{MSE}}/v^\pi=O(\sqrt{H^2}/H)=O(1)$, which is consistent with the horizontal line when $H$ is large. Moreover, for SMIS $\sqrt{\text{MSE}}/v^\pi=O(\sqrt{H^3}/H)=O(\sqrt{H})$, so after taking the $\log(\cdot)$ we should have asymptotic linear trend with coefficient $1/2$.  The red line in Figure~\ref{fig:different_H} empirically verifies this result. More empirical study discussions are deferred to Appendix~\ref{sec:add_empirical}.

	%	\includegraphics[height=1.2in]{TMIS_non-stationary_different_H.pdf}

	%	\includegraphics[height=1.2in]{TMIS_non-stationary_different_n.pdf}

%between [0,1] and the state transition models are time-varying with a finite horizon H. The agent starts at State 1. At every step, the environment samples a random parameter p ∈ [0.5/H, 0.5 − 0.5/H]. Any agent in State 1 will transition to State 0 if and only if it samples an action between [p − 0.5/H, p + 0.5/H]. On the other hand, State 0 is a sinking state. The agent collects rewards at State 0 in the latter half of the steps (t ≥ H/2). Thus, the agent wants to transition to State 0, but the transition probability is inversely proportional to the horizon H for uniform action policies. We pick the behavior policy to be uniform on [0, 1] and the target policy to be uniform on [0, 0.5] with 95% total probability and 5% chance uniformly distributed on [0.5, 1].

\section{Discussion}\label{sec:discussion} 
% !TEX root = mis_tabular_aistats.tex

\noindent\textbf{From off-policy evaluation to offline learning.} A real offline reinforcement learning system is equipped with both offline learning algorithms and off-policy evaluation algorithms. The decision maker should first run the offline learning algorithm to find a near optimal policy and then use off-policy evaluation methods to check if the obtained policy is good enough. Under our tabular MDP setting, we point out it is possible to find a $\epsilon$-optimal policy in near optimal time and sample complexity $O(H^3SA/\epsilon^2)$using the $Q$-value iteration (QVI) based algorithm designed by \citet{sidford2018near}. Their QVI algorithm assumes a generative model which can provide independent sample of the next state $s^\prime$ given any current state-action $(s,a)$. At a first glance, this assumption seems too strong for offline learning since we cannot force the agent to stay in any arbitrary location. In fact, the Assumption~\ref{assume2} on $\mu$ actually reveals that the underlying logging policy can be considered as the surrogate of the generative model. As $n$ goes large, the visitation frequency of any $(s_t,a_t)$ will be large enough with high probability, as guaranteed by Multiplicative Chernoff bound.
      
\noindent\textbf{From off-policy evaluation to uniform off-policy evaluation.} The high probability result achieves $\widetilde{O}(\sqrt{H^2SA/n})$ complexity. Following this discovery line, then it is natural to ask whether uniform convergence over a class of policies (\emph{e.g.} all deterministic policies) can be achieved with optimal sample complexity. This problem is interesting since it will guarantee the strong performance of off-policy evaluation methods over all policies in certain policy class $\prod$. By a direct application of union bound, we can obtain the following result:
\begin{theorem}\label{thm:unifrom_convergence}
	Let $\prod$ contains all the deterministic $H$-step policies. Then under the same condition as Theorem~\ref{thm:high_probability}, the data splitting \TMIS{} satisfies:
	\[
	\sup_{\pi\in\prod}	|\widehat{v}_{\mathrm{split}}^\pi -  v^\pi|\leq
	\widetilde{O}\big(\sqrt{\frac{H^3S^2A}{n}  }\big),
	\]
	with probability $1-\delta$. 
\end{theorem}  
The uniform convergence bound implies that the empirical best policy $\hat{\pi} = \argmax_{\pi} \widehat{v}_{\mathrm{split}}^\pi $ is within $\epsilon = \sqrt{\frac{H^3S^2A}{n}}$ of the optimal policy. This matches the sample complexity lower bound for learning the optimal policy  \citep{azar2013minimax} in all parameters except a factor of $S$. 
%The $H^3$ dependence is usually required for tabular MDP learning and we conjecture the underlying optimal sample complexity is $O(H^3SA/\epsilon^2)$. Note the dependence $S^2A$ here is not very interesting since when you have $S^2A$ dependence, it is sufficient for you estimate the transition dynamic $P(s'|s,a)$ up to a high accuracy. Therefore uniform off-policy evaluation problem remains open and we conjecture by combing the idea of martingale concentration inequalities and Rademacher complexity might be helpful in solving this problem.

\noindent\textbf{Open problem: $H^3$ vs $H^2$ in the infinite $\cA$ setting.} 
Finally, we note that the conjecture posed by \citet{xie2019towards} remains unsolved. The result we presented in this paper leverage the fact that we can estimate the parameters of the MDP model. In the infinite $\cA$ case, we can never observe any $(s,a)$ pairs more than once, hence not able to estimate the transition dynamics or the expected reward. The minimax lower bound in \citep{wang2017optimal} (for the contextual bandit setting) already establishes that the Cramer-Rao lower bound is not achievable in this setting even if $H=1$ and $S=1$.  It remains open whether $H^3$ is required.

\section{Conclusion}
In this paper, we propose and analyze a new marginalized importance sampling estimator for the off-policy evaluation (OPE) problem under the episodic tabular Markov decision process model. We show that the estimator is has a finite sample error bound that matches the exact Cramer-Rao lower bound up to low-order factors. We also provide an extension with high probability error bound.  To the best of our knowledge, these results are the first of their kinds. Future work includes resolving the open problems mentioned before and generalizing the results to more practical settings.

%\red{Talk about how CR lower bound might not be asymptotically achievable in some cases.}

%\section*{Acknowledgment}

\bibliographystyle{apa-good.bst}
\bibliography{stat_rl.bib}

\appendix
% \onecolumn

\clearpage
\begin{center}
	% {\LARGE \textbf{Appendix}}
	{\LARGE Appendix}
\end{center}

% !TEX root = mis_tabular_aistats.tex

\section{Concentration inequalities and other technical lemmas}
\begin{lemma}[Bernstein’s Inequality \citep{sridharan2002gentle} ]\label{lem:bernstein_ineq}
	Let $x_1,...,x_n$ be independent bounded random variables such that $\E[x_i]=0$ and $|x_i|\leq \xi$ with probability $1$. Let $\sigma^2 = \frac{1}{n}\sum_{i=1}^n \mathrm{Var}[x_i]$, then for any $\epsilon >0$ we have 
	$$
	\P\left( \frac{1}{n}\sum_{i=1}^nx_i\geq \epsilon\right) \leq e^{-\frac{n\epsilon^2}{2\sigma^2+2\xi\epsilon/3}}.
	$$
\end{lemma}

\begin{lemma}[Multiplicative Chernoff bound \citep{chernoff1952measure} ]\label{lem:chernoff_multiplicative}
	Let $X$ be a Binomial random variable with parameter $p,n$. For any $\delta>0$, we have that 
	$$
	\P[X <(1-\delta)pn] < \left( \frac{e^{-\delta}}{(1-\delta)^{1-\delta}} \right)^{np}.  
	$$
\end{lemma}
A slightly weaker bound that suffices for our propose is the following:
$$
\P[X < (1-\delta)pn] <  e^{-\frac{\delta^2 pn}{2}}.
$$

\section{Proof of the main theorem}
\label{sec:full_theo}

To analyze the MSE upper bound $\mathbb{E}_\mu[(\widehat{v}^\pi_{\text{TMIS}}-v^\pi)^2]$, we create a fictitious surrogate $\widetilde{v}^\pi_{\text{TMIS}}$, which is an unbiased version of $\widehat{v}^\pi_{\text{TMIS}}$. A few auxiliary lemmas are first presented and Bellman equations are used for deriving variance decomposition in a recursive way. Second order moment of marginalized state distribution $\widetilde{d}^\pi_t$ can then be bounded by analyzing its variance. 

\subsection{Fictitious tabular MIS estimator.}\label{sec:Fictitious}

%Also, let $\mathcal{D}$ to denote the historical data containing in total $n$ episode trajectories and the corresponding $\mathcal{D}_h=\{(s_t^{(i)},a_t^{(i)},r_t^{(i)}):i\in[n],t\leq h\}$ to be roll-in realization of $n$ trajectories up to step $h$. 
The fictitious estimator\footnote{We replcace the notation of  $\widetilde{v}^\pi_\text{TMIS}$ with just $\widetilde{v}^\pi$ throughout the proof. $\widetilde{v}^\pi$ always denotes fictitious tabular MIS estimator.} $\widetilde{v}^\pi$ fills in the gap of state-action location $(s_t,a_t)$ of the true estimator $\widehat{v}^\pi$ where $n_{s_t,a_t}=0$. Specifically, it replaces every component in $\widehat{v}^\pi$ with a fictitious counterpart, \emph{i.e.} $\widetilde{v}^\pi :=  \sum_{t=1}^H \langle \widetilde{d}_t^\pi,  \widetilde{r}_t^\pi\rangle$, with $\widetilde{d}_t^\pi =  \widetilde{P}_{t}^\pi   \widetilde{d}_{t-1}^\pi$ and $\widetilde{P}^\pi_t(s_t|s_{t-1})=\sum_{a_{t-1}}\widetilde{P}_t(s_t|s_{t-1},a_{t-1})\pi(a_{t-1}|s_{t-1}),$
$\widetilde{r}^\pi_t(s_t)=\sum_{a_t}\widetilde{r}_t(s_t,a_t)\pi(a_t|s_t).$ In particular, let $E_t$ denotes the event $\{n_{s_t,a_t}\geq nd^\mu_t(s_t,a_t)(1-\theta)\}$\footnote{More rigorously, $E_t$ depends on the specific pair $s_t,a_t$ and should be written as $E_t(s_t,a_t)$. However, for brevity we just use $E_t$ and this notation should be clear in each context.}, then 
\begin{align*}
\widetilde{r}_t(s_t,a_t)&=\widehat{r}_t(s_t,a_t)\mathbf{1}(E_t)+{r}_t(s_t,a_t)\mathbf{1}(E_t^c)\\
\widetilde{P}_{t+1}(\cdot|s_t,a_t)&=\widehat{P}_{t+1}(\cdot|s_t,a_t)\mathbf{1}(E_t)+{P}_{t+1}(\cdot|s_t,a_t)\mathbf{1}(E_t^c).
\end{align*}
where $0<\theta<1$ is a parameter that we will choose later.

The name "fictitious" comes from the fact that $\widetilde{v}^\pi$ is \emph{not implementable} using the data\footnote{It depends on unknown information such as $d_t^\mu$, $P_{t,t-1}^\pi$, exact conditional expectation of the reward $r_t^\pi$ and so on.}, but it creates a bridge between $\widehat{v}^\pi$ and $v^\pi$ because of its unbiasedness, see Lemma~\ref{lem:unbiasedness_fictitious}. Also, for simplicity of the proof, throughout the rest of the paper we denote:
$
\mathcal{D}_t := \left\{s_{1:{t}}^{(i)}, a_{1:{t}}^{(i)},r_{1:{t-1}}^{(i)} \right\}_{i=1}^n.
$
Also, in the base case, we denote $\mathcal{D}_1 := \left\{s_{1}^{(i)},a_{1}^{(i)}\right\}_{i=1}^n$ and that 
$
r_{t}^\pi(s_t) := \E_\pi[ r_{t}^{(1)} |  s_t^{(1)} = s_t ]=\sum_{a_t}\E[ r_{t}^{(1)} |  s_t^{(1)} = s_t,a_t^{(1)}=a_t ]\pi(a_t|s_t):=\sum_{a_t}r_t(s_t,a_t)\pi(a_t|s_t)
$.
Then we have the following preliminary auxiliary lemmas. 

\begin{lemma}\label{lem:unbiasedness_given_data}
	$\widetilde{d}^\pi_t$ and $\widetilde{r}^\pi_{t-1}$ are deterministic given $\mathcal{D}_t$. Moreover, given $\mathcal{D}_t$, $\widetilde{P}_{t+1,t}^\pi$ is unbiased of $P_{t+1,t}^\pi$ and $\widetilde{r}^\pi_t$ is unbiased of $r^\pi_t$.
\end{lemma}
\begin{proof}[Proof of Lemma~\ref{lem:unbiasedness_given_data}]
	By construction of the estimator, $\widetilde{d}^\pi_t$ and $\widetilde{r}^\pi_{t-1}$ only depend on $\mathcal{D}_t$, therefore $\widetilde{d}^\pi_t$ and $\widetilde{r}^\pi_{t-1}$ given $\mathcal{D}_t$ are constants. For the second argument, we have $\forall s_t,s_{t+1}$,
	\begin{align*}
	&\mathbb{E}[\widetilde{P}^\pi_{t+1,t}(s_{t+1}|s_t)|\mathcal{D}_{t}]=\sum_{a_t}\mathbb{E}[\widetilde{P}_{t+1,t}(s_{t+1}|s_t,a_t)|\mathcal{D}_t]\pi(a_t|s_t)\\
	&=\sum_{a_t}\bigg(\mathbf{1}(E_t)\mathbb{E}[\widehat{P}_{t+1,t}(s_{t+1}|s_t,a_t)|\mathcal{D}_t]+\mathbf{1}(E^c_t){P}_{t+1,t}(s_{t+1}|s_t,a_t)\bigg)\pi(a_t|s_t)\\
	&=\sum_{a_t}\bigg(\mathbf{1}(E_t){P}_{t+1,t}(s_{t+1}|s_t,a_t)+\mathbf{1}(E^c_t){P}_{t+1,t}(s_{t+1}|s_t,a_t)\bigg)\pi(a_t|s_t)\\
	&=\sum_{a_t}{P}_{t+1,t}(s_{t+1}|s_t,a_t)\pi(a_t|s_t)=P^\pi_{t+1,t}(s_{t+1}|s_t),\\
	\end{align*}
	where the third equal sign comes from the fact that conditional on $E_t$, $\hat{P}(s_{t+1}|s_t,a_t)$ --- the empirical mean --- is unbiased. The result about $\tilde{r}_t^\pi$ can be derived using a similar fashion.
\end{proof}

Using Lemma~\ref{lem:unbiasedness_given_data}, we can derive the following  recursions for expectation and variance:

\begin{lemma}\label{lem:inter_recursion} 
	For $h=1,...,H$, we have 
	\begin{align}
	\label{eq:exp_recursive}
\E\left[\langle\widetilde{d}_{h}^{\pi},  V_h^{\pi}\rangle +  \sum_{t=1}^{h-1}  \langle\widetilde{d}_{t}^{\pi},  \widetilde{r}_t^{\pi}\rangle  \middle|  \mathcal{D}_{h-1} \right] &=\langle \widetilde{d}_{h-1}^\pi, V_{h-1}^\pi \rangle + \sum_{t=1}^{h-2}  \langle\widetilde{d}_{t}^{\pi},  \widetilde{r}_t^{\pi}\rangle,\\
\label{eq:var_recursive}
\Var\left[  \langle \widetilde{d}_{h+1}^\pi, V_{h+1}^\pi\rangle +  \sum_{t=1}^h \langle\widetilde{d}_{t}^{\pi},  \widetilde{r}_t^{\pi} \rangle\right] &=  \E\left[ \Var\left[  \langle \widetilde{d}_{h+1}^\pi, V_{h+1}^\pi\rangle  +  \langle \widetilde{d}_h^\pi, \widetilde{r}_h^\pi\rangle \middle| \mathcal{D}_{h} \right] \right] 
+ \Var\left[  \langle \widetilde{d}_{h}^\pi, V_{h}^\pi\rangle +  \sum_{t=1}^{h-1} \langle\widetilde{d}_{t}^{\pi},  \widetilde{r}_t^{\pi} \rangle \right]
\end{align}
\end{lemma}
\begin{proof}
	The proof of Lemma~\ref{lem:inter_recursion} can be found in Lemma~B.2 and Lemma~4.1 in \citet{xie2019towards} by coupling the standard Bellman equation:
	\begin{equation}\label{eq:bellman_eq}
		V_h^\pi =   r_{h}^\pi  + [P_{h+1,h}^\pi]^T V_{h+1}^\pi
	\end{equation}
	with the total law of expectations and the total law of variances.
\end{proof}

\begin{lemma}[Boundedness of Tabular MIS estimators]\label{lem:bounded_estimator} 
$0\leq \widehat{v}^\pi\leq HR_{\max}$, $0\leq \widetilde{v}^\pi\leq HR_{\max}$.
\end{lemma}
	\begin{proof}
	we show $\widehat{P}^\pi_t(\cdot|s_{t-1})$ is a (degenerated) probability distribution for all $t,s_{t-1}$. 

	\begin{equation}\label{eq:degen}
	\begin{aligned}
	\sum_{s_t}\widehat{P}^\pi_t(s_t|s_{t-1})&=\sum_{s_t}\sum_{a_{t-1}}\widehat{P}_t(s_t|s_{t-1},a_{t-1})\pi(a_{t-1}|s_{t-1})\\
	&=\sum_{a_{t-1}}\sum_{s_t}\widehat{P}_t(s_t|s_{t-1},a_{t-1})\pi(a_{t-1}|s_{t-1})\qquad \text{This is since $|\mathcal{A}|,|\mathcal{S}|<\infty$}\\
	&=\sum_{a_{t-1}}\sum_{s_t}\frac{n_{s_t,s_{t-1},a_{t-1}}}{n_{s_{t-1},a_{t-1}}}\pi(a_{t-1}|s_{t-1})\\
	&\leq \sum_{a_{t-1}}\pi(a_{t-1}|s_{t-1})=1\\
	\end{aligned}
	\end{equation}
	
	The last line is inequality since $\widehat{P}_t(s_t|s_{t-1},a_{t-1})=0$ when $n_{s_{t-1},a_{t-1}}=0$. Following the same logic, it is easy to show $\widetilde{P}^\pi_t(\cdot|s_{t-1})$ is a non-degenerated probability distribution. 
	
	Next note $\sum_{s_1}\widehat{d}^\pi_1(s_1)=\sum_{s_1}\widehat{d}^\mu_1(s_1)=\sum_{s_1}\frac{n_{s_1}}{n}=1$. Suppose $\widehat{d}^\pi_{t-1}(\cdot)$ is a (degenerated) probability distribution, then from $\widehat{d}^\pi_t=\widehat{P}^\pi_t\widehat{d}^\pi_{t-1}$ and \eqref{eq:degen}, by induction we know $\widehat{d}^\pi_t(\cdot)$ is a (degenerated) probability distribution for all $t$. 
	
	Using Assumption~\ref{assume1}, it is easy to show $\hat{r}^\pi_t(s_t)\leq R_{\max}$ for all $s_t$, then combining all results above we have $\widehat{v}^\pi :=  \sum_{t=1}^H \langle \widehat{d}_t^\pi,  \widehat{r}_t^\pi\rangle\leq HR_{\max}$. Similarly, $\widetilde{v}^\pi \leq HR_{\max}$.

	\end{proof}

The boundedness of \TMIS{} cannot be inherited by the \SMIS{} since $\widehat{v}^\pi_{\text{SMIS}}$ explicitly uses importance weights and there is no reason for it to be less than $HR_{\max}$.  As a result, we do not need an extra projection step for our estimation to be valid (see \citet{xie2019towards} Lemma~B.1). Thanks to the following lemma, throughout the rest of the analysis we only need to consider $\widetilde{v}^\pi$. 

\begin{lemma}\label{lem:fictitious_approximation}
	Let $\widehat{v}^\pi$ be the \TMIS{} and $\widetilde{v}^\pi$ be the fictitious version of TMIS we described above with parameter $\theta$. Then the MSE of the TMIS and fictitious TMIS satisfies
	%	|\E[\widehat{v}^\pi] - v^\pi|  \leq \P[\exists t,s_t, \; \text{s.t.} , \; n_{s_t} = 0]  H R_{\max}
	%	$$
	$$
	\E[(\widehat{v}^\pi- v^\pi)^2]  \leq \E[(\widetilde{v}^\pi- v^\pi)^2] + 3H^3SAR_{\max}^2 e^{-\frac{\theta^2n \min_{t,s_t,a_t}d^{\mu}_t(s_t,a_t)}{2}}
	%\left(\frac{2}{e}\right)^{ \frac{n \min_{t,s_t}d^{\mu}_t(s_t)}{2}}.
	$$
\end{lemma}
\begin{proof}[Proof of Lemma~\ref{lem:fictitious_approximation}]
	Define $E:=\{\exists t, s_t,a_t \;\text{s.t.}\; n_{s_t,a_t} <  n d_t^\mu(s_t,a_t)(1-\theta)  \}$.
	Similarly to Lemma~B.1 in the appendix of \citet{xie2019towards}, we have 
	%we have that $$(\cP \widehat{v}^\pi  - v^\pi)^2 \leq (\cP_E \widehat{v}^\pi  - v^\pi)^2.$$
\begin{align*}
&\E[(\widehat{v}^\pi  - v^\pi)^2] \leq  \E[( \widehat{v}^\pi- v^\pi)^2] = \E[( \widehat{v}^\pi - \widetilde{v}^\pi)^2] +2\E[( \widehat{v}^\pi - \widetilde{v}^\pi)(\widetilde{v}^\pi - v^\pi)]  + \E[( \widetilde{v}^\pi - v^\pi)^2]\\
=& \P[E] \E\big[(\widehat{v}^\pi - \widetilde{v}^\pi)^2+ 2(\widehat{v}^\pi - \widetilde{v}^\pi)(\widetilde{v}^\pi - v^\pi) \big|E\big] + \P[E^c] \cdot 0  + \E[( \widetilde{v}^\pi -  v^\pi)^2]\\
\leq& 3 \P[E]H^2R_{\max}^2+ \E[(\widetilde{v}^\pi - v^\pi)^2], %= 3 \P[E]R_{\max}^2+ \Var[\widetilde{v}^\pi].
\end{align*}

	where the last inequality uses Lemma~\ref{lem:bounded_estimator}. Then combining the multiplicative Chernoff bound (Lemma~\ref{lem:chernoff_multiplicative} in the Appendix) and a union bound over each $t$,$s_t$ and $a_t$, we get that
	\begin{align*}
	\P[E] &\leq \sum_{t}\sum_{s_t}\sum_{a_t} \P[n_{s_t,a_t} < n d^{\mu}_t(s_t,a_t)(1-\theta)] \leq HSA e^{-\frac{\theta^2 n \min_{t,s_t,a_t} d^{\mu}_t(s_t,a_t)}{2}},
	\end{align*}
	which provides the stated result.
	%	with $\theta=0.5$, we get that
	%	$$
	%	\P\left[ n_{s_t} < \frac{n d^{\mu}_t(s_t)}{2} \right] \leq  \left(\frac{2}{e}\right)^{ \frac{n d^{\mu}_t(s_t)}{2}} \leq \left(\frac{2}{e}\right)^{ \frac{n \min_{t,s_t}d^{\mu}_t(s_t)}{2}}.
	%	$$
	%	By a union bound over each $t$ and $s_t$, we have
	%	\begin{align*}
	%	\P[E] &\leq \sum_{t}\sum_{s_t} \P[n_{s_t,t} < \frac{n d^{\mu}_t(s_t)}{2}] \leq HS \left(\frac{2}{e}\right)^{ \frac{n \min_{t,s_t}d^{\mu}_t(s_t)}{2}}.
	%	\end{align*}
	%	as stated.
\end{proof}

Lemma~\ref{lem:fictitious_approximation} tells that MSE of two TMISs differs by a quantity $3H^3SAR_{\max}^2 e^{-\frac{\theta^2n \min_{t,s_t,a_t}d^{\mu}_t(s_t,a_t)}{2}}$ and this illustrates that the gap between two MSE's can be sufficiently small as long as  $n \geq \frac{\mathrm{polylog}(S,A,H,n)}{\min_{t,s_t,a_t}d^\mu_{t}(s_t,a_t)}$.

\subsection{Variance and Bias of Fictitious tabular MIS estimator.}

\begin{lemma}[\citet{xie2019towards} Lemma~B.2]\label{lem:unbiasedness_fictitious} \TMIS{} is unbiased: $\E[\widetilde{v}^\pi] = v^\pi$ for all $0<\theta<1$.
\end{lemma}

\begin{lemma}[Variance decomposition]\label{lem:var_decomp_fictitious}
	% 	\begin{align}
	% 	\Var[\widetilde{v}^\pi]  \leq & \frac{\|V_1^\pi\|^2_{d_1(\cdot)}}{n} + 2\sum_{h=1}^H \sum_{s_h}  \E\left[   \frac{\widetilde{d}_{h}^\pi(s_h)^2}{n_{s_h}}   \mathbf{1}_{\{ n_{s_h}\geq \frac{n d_h^{\mu}(s_h)}{2} \}}\right]   \sum_{a_h}\frac{\pi(a_h|s_h)^2}{\mu(a_h|s_h)} \big(\| V_{h+1}^{\pi}\|_{P_{h+1,h}(\cdot|s_h,a_h)}^2 \\
	% 	& + \sigma^2(s_h,a_h) + r_h(s_h,a_h)^2\big).
	% 	\end{align}
	\begin{equation}\label{var_decomp}	
	\begin{aligned}
	\Var[\widetilde{v}^\pi] =& \frac{\Var[V_{1}^\pi(s_1^{(1)})]}{n} 
	\\&+  \sum_{h=1}^H \sum_{s_h}\sum_{a_h}  \E\left[\frac{\widetilde{d}_h^\pi(s_h)^2}{n_{s_h,a_h}} \mathbf{1}(E_h)\right]  \pi(a_h|s_h)^2  \Var\left[ (V_{h+1}^\pi(s_{h+1}^{(1)}) +  r_h^{(1)})\middle| s_{h}^{(1)}=s_h,a_h^{(1)}=a_h\right].
	\end{aligned}
	\end{equation}
	where
	%as defined in the proof of Lemma~\ref{lem:unbiasedness_fictitious},  $$\text{Data}_t := \left\{s_{1:{t}}^{(i)}, a_{1:{t-1}}^{(i)},r_{1:{t-1}}^{(i)} \right\}_{i=1}^n$$ and 
	$V_t^\pi(s_t)$ denotes the value function under $\pi$ which satisfies the Bellman equation 
	$$ V_t^\pi(s_t)=  r_t^\pi(s_t)  + \sum_{s_{t+1}} P^\pi_t(s_{t+1}| s_t)   V_{t+1}^\pi(s_{t+1}).$$
	% 	and we used $\|x\|_{w}^2:=\sum_i w[i] x[i]^2$ to denote squared weighted Euclidean norm.
\end{lemma}
% \begin{remark}
% The decomposition of variance is very interpretable. The first part of the variance is coming from estimating the initial state. The second part is coming from the conditional variance of estimating $P_{t+1,t}^{\pi}(s_t)$ and $r_t^\pi(s_t)$ using importance sampling over $a_t$.
% %The second part ($\|V_{h+1}\|_{P_{h+1,h}(\cdot|s_h,a_h)}^2$) is coming from the conditional variance of estimating $P_{t,t-1}^{\pi}$ using importance sampling over $a_t$ given all observations up to $t-1$. The third part ( $\sigma^2(s_h,a_h) + r_h(s_h,a_h)^2$) is coming from the conditional variance of estimating $r_{t}^\pi$ using importance sampling over $a_t$ given all observations up to time $t$.
% \end{remark}
\begin{remark}
	Note even though the construction of TMIS and SMIS are different, both fictitious estimators are unbiased for $v^\pi$. Therefore the MSE of MIS estimators are dominated by the variance of the fictitious estimators. Comparing Lemma~\ref{lem:var_decomp_fictitious} with Lemma~4.1 in \citet{xie2019towards} we can see our \TMIS{} achieves a lower bound, and it is essentially asymptotic optimal, as explained by Remark~\ref{remark:asym}.
\end{remark}

\begin{proof}[Proof of Lemma~\ref{lem:var_decomp_fictitious}]
	The proof relies on applying Lemma~\ref{lem:inter_recursion} in a recursive way. One key observation is

	To begin with the following variance decomposition, which applies \eqref{eq:var_recursive} recursively.
	
	\begin{align*}
	\Var[\widetilde{v}^{\pi} ]  = & \E \Var[\widetilde{v}^{\pi} |  \mathcal{D}_H ]  + \Var[ \E[ \widetilde{v}^{\pi} | \mathcal{D}_H]] \\
	=& \E\left[ \Var[ \langle\widetilde{d}_{H}^{\pi},  \widetilde{r}_H^{\pi} \rangle   |  \mathcal{D}_H ]\right]    +  \Var[  \E[  \langle\widetilde{d}_{H}^{\pi},  \widetilde{r}_H^{\pi} \rangle  | \mathcal{D}_{H}]  + \sum_{t=1}^{H-1} \langle\widetilde{d}_{t}^{\pi},  \widetilde{r}_t^{\pi} \rangle  ] \\
	=&\E\left[ \Var[ \langle\widetilde{d}_{H}^{\pi},  \widetilde{r}_H^{\pi} \rangle   |  \mathcal{D}_H ]\right]   +  \Var[  \langle\widetilde{d}_{H}^{\pi},  r_H^{\pi} \rangle  + \sum_{t=1}^{H-1} \langle\widetilde{d}_{t}^{\pi},  \widetilde{r}_t^{\pi} \rangle   ] \\
	=&\E\left[ \Var[ \langle\widetilde{d}_{H}^{\pi},  \widetilde{r}_H^{\pi} \rangle   | \mathcal{D}_H ]\right]    +  \Var[  \langle\widetilde{d}_{H}^{\pi},  V_H^{\pi} \rangle  + \sum_{t=1}^{H-1} \langle\widetilde{d}_{t}^{\pi},  \widetilde{r}_t^{\pi} \rangle   ] \\
	=&\E\left[ \Var[ \langle\widetilde{d}_{H}^{\pi},  \widetilde{r}_H^{\pi} \rangle   |  \mathcal{D}_H ]\right]  + \E\left[ \Var\left[  \langle \widetilde{d}_{H}^\pi, V_{H}^\pi\rangle  +  \langle \widetilde{d}_{H-1}^\pi, \widetilde{r}_{H-1}^\pi\rangle \middle| \mathcal{D}_{H-1} \right] \right] \\
	&+ \Var\left[  \langle \widetilde{d}_{H-1}^\pi, V_{H-1}^\pi\rangle +  \sum_{t=1}^{H-2} \langle\widetilde{d}_{t}^{\pi},  \widetilde{r}_t^{\pi} \rangle \right]=...\\
	=&\E\left[ \Var[ \langle\widetilde{d}_{H}^{\pi},  \widetilde{r}_H^{\pi} \rangle   |  \mathcal{D}_H ]\right]  +\sum_{h=1}^{H-1} \E\left[ \Var\left[  \langle \widetilde{d}_{h+1}^\pi, V_{h+1}^\pi\rangle  +  \langle \widetilde{d}_{h}^\pi, \widetilde{r}_{h}^\pi\rangle \middle| \mathcal{D}_{h} \right] \right] +\Var\left[  \langle \widetilde{d}_{1}^\pi, V_{1}^\pi\rangle \right]\\
	%\leq&\Var\left[  \langle \widetilde{d}_{1}^\pi, V_{1}^\pi\rangle \right] + 2 \sum_{h=2}^H\E\left[ \Var[ \langle\widetilde{d}_{h}^{\pi},  \widetilde{r}_h^{\pi} \rangle   |  \text{Data}_h ]\right]  + 2 \sum_{h=2}^H
	%	=&\E \Var[ \langle\widetilde{d}_{H}^{\pi},  \widetilde{r}_H^{\pi} \rangle   |  \text{Data}_{1:H-1} ]   +  \Var[ \sum_{t=1}^{H-2} \langle\widetilde{d}_{t}^{\pi},  \widetilde{r}_t^{\pi}\rangle +  \langle  \widetilde{d}_{H-1}^{\pi} , \widetilde{r}_{H-1}^{\pi} 
	%	+  \P_H[r_H^\pi, \pi(a_{H-1}|s_{H-1}), \cdot]  \rangle]\\
	%	&= \sum_{t=H-1}^H   \E \Var[ \langle\widetilde{d}_{t}^{\pi},  \widetilde{r}_t^{\pi} + V_{t+1}^\pi \rangle   |  \text{Data}_{1:t-1} ]  +  \Var[ \sum_{t=1}^{H-3} \langle\widetilde{d}_{t}^{\pi},  \widetilde{r}_t^{\pi}\rangle  +  \langle  \widetilde{d}_{H-2}^{\pi} , \widetilde{r}_{H-2}^{\pi} +  \P_{H-1}[V_{H-1}^\pi, \pi(a_{H-2}|s_{H-2}), \cdot]  \rangle ]\\
	%	&= \sum_{t=1}^H   \E \Var[ \langle\widetilde{d}_{t}^{\pi},  \widetilde{r}_t^{\pi} + V_{t+1}^\pi \rangle   |  \text{Data}_{1:t-1} ] \\
	%	&\leq 2\sum_{t=1}^H \E \Var[ \langle\widetilde{d}_{t}^{\pi},  \widetilde{r}_t^{\pi} \rangle   |  \text{Data}_{1:t-1} ]  +   2\sum_{t=1}^H \E \Var[ \langle\widetilde{d}_{t}^{\pi},  V_{t+1}^\pi \rangle |  \text{Data}_{1:t-1} ].
	\end{align*}

	Now let us analyze $\E\left[ \Var\left[  \langle \widetilde{d}_{h+1}^\pi, V_{h+1}^\pi\rangle  +  \langle \widetilde{d}_{h}^\pi, \widetilde{r}_{h}^\pi\rangle \middle| \mathcal{D}_{h} \right] \right]$. Note $\widetilde{P}_{h+1,h}^\pi(\cdot,s_h)$ and $\widetilde{r}_h^\pi(s_h)$ for each $s_h$  are conditionally independent given $\mathcal{D}_h$, since $\mathcal{D}_h$ partitions the $n$ episodes into $S$ disjoint sets according to the states  $s_h^{(i)}$ at time $h$. Similarly, $\tilde{P}_{h+1}(\cdot|s_h,a_h)$ and $\tilde{r}^\pi_h(s_h,a_h)$ for each $(s_h,a_h)$ are also conditionally independent given $\mathcal{D}_h$. These observations imply:

	\begin{equation}	
	\begin{aligned}
	&\E\left[ \Var\left[  \langle \widetilde{d}_{h+1}^\pi, V_{h+1}^\pi\rangle  +  \langle \widetilde{d}_h^\pi, \widetilde{r}_h^\pi\rangle \middle| \mathcal{D}_{h} \right] \right] \\
	=&\E\left[ \sum_{s_h} \Var\left[  \widetilde{d}_{h}^\pi(s_h)\langle\widetilde{P}^\pi_{h+1,h}(\cdot,s_h) , V_{h+1}^\pi\rangle  +  \widetilde{d}_h^\pi(s_h)\cdot \widetilde{r}_h^\pi(s_h) \middle| \mathcal{D}_{h} \right] \right] \\
	=&\E\left[ \sum_{s_h} \widetilde{d}_{h}^{\pi 2}(s_h) \Var\left[  \sum_{a_h}\langle\widetilde{P}_{h+1}(\cdot|s_h,a_h)\cdot\pi(a_h|s_h) , V_{h+1}^\pi\rangle  +  \sum_{a_h}\widetilde{r}_h(s_h,a_h)\cdot\pi(a_h|s_h) \middle| \mathcal{D}_{h} \right] \right] \\
	=&\E\left[ \sum_{s_h} \widetilde{d}_{h}^{\pi }(s_h)^2 \sum_{a_h}\pi(a_h|s_h)^2\Var\left[  \langle\widetilde{P}_{h+1}(\cdot|s_h,a_h) , V_{h+1}^\pi\rangle  +  \widetilde{r}_h(s_h,a_h) \middle| \mathcal{D}_{h} \right] \right] \\
	=&\E\left[ \sum_{s_h}  \widetilde{d}_h^\pi(s_h)^2\sum_{a_h}\pi(a_h|s_h)^2  \mathbf{1}(E_t)\Var\left[  \frac{1}{n_{s_h,a_h}}\sum_{i  | s_h^{(i)}=s_h,a_h^{(i)}=a_h }   (V_{h+1}^\pi(s_{h+1}^{(i)}) +  r_h^{(i)})  \middle| \mathcal{D}_{h} \right] \right]\\
	=&\E\left[ \sum_{s_h}  \widetilde{d}_h^\pi(s_h)^2\sum_{a_h}\pi(a_h|s_h)^2 \cdot \frac{\mathbf{1}(E_t)}{n_{s_h,a_h}}\cdot\Var\left[    (V_{h+1}^\pi(s_{h+1}^{(i)}) +  r_h^{(i)})  \middle| s^{(i)}_h=s_h,a^{(i)}_h=a_h \right] \right]\\
	=& \sum_{s_h}  \sum_{a_h}\pi(a_h|s_h)^2 \cdot \E\left[\frac{\widetilde{d}_h^\pi(s_h)^2}{n_{s_h,a_h}}\cdot\mathbf{1}(E_t)\right]\cdot\Var\left[    (V_{h+1}^\pi(s_{h+1}^{(i)}) +  r_h^{(i)})  \middle| s^{(i)}_h=s_h,a^{(i)}_h=a_h \right] .\label{eq:var_recursion_per_step}
	%=&\sum_{s_h}  \E\left[\frac{\widetilde{d}_h^\pi(s_h)^2}{n_{s_h}} \mathbf{1}(E_h)\right] \Bigg(\E_\mu\left[  [\rho_h^{(1)}]^2(V_{h+1}(s_{h+1}^{(1)})+r_h^{(1)})^2 \middle| s_h^{(1)}=s_h \right]-    V_{h}(s_h)^2\Bigg)
	\end{aligned}
	\end{equation}
	The second line and the fourth line use the conditional independence for $s_t$ and $(s_t,a_t)$ respectively. The fifth line uses that when $n_{s_h,a_h} < nd_h^{\mu}(s_h,a_h)(1-\theta)$, the conditional variance is $0$. The sixth line uses the fact that episodes are iid.

	Plug \eqref{eq:var_recursion_per_step} into the above variance decomposition and uses $V_{H+1}=0$, we finally get
	\begin{align*}
	\Var[\widetilde{v}^\pi] =& \frac{\Var[V_{1}^\pi(s_1^{(1)})]}{n} 
	\\&+  \sum_{h=1}^H \sum_{s_h}\sum_{a_h}  \E\left[\frac{\widetilde{d}_h^\pi(s_h)^2}{n_{s_h,a_h}} \mathbf{1}(E_h)\right]  \pi(a_h|s_h)^2  \Var\left[ (V_{h+1}^\pi(s_{h+1}^{(1)}) +  r_h^{(1)})\middle| s_{h}^{(1)}=s_h,a_h^{(1)}=a_h\right].
	\end{align*}
	%	Finally, apply \eqref{eq:var_decomp_fictitious} with the two bounds  \eqref{eq:var_fictitious_bound1} and \eqref{eq:var_fictitious_bound2}, we get
	%	\begin{align*}
	%	\Var[\widetilde{v}^\pi]  \leq  \frac{\|V_1^\pi\|^2_{d_1(\cdot)}}{n} + 2\sum_{h=1}^H \sum_{s_h}  \E\left[   \frac{\widetilde{d}_{h}^\pi(s_h)^2}{n_{s_h}}   \mathbf{1}_{\{ n_{s_h}\geq \frac{n d_h^{\mu}(s_h)}{2} \}}\right]   \sum_{a_h}\frac{\pi(a_h|s_h)^2}{\mu(a_h|s_h)} (\| V_{h+1}^{\pi}\|_{P_{h+1,h}(\cdot|s_h,a_h)}^2 \\
	%	+ \sigma^2(s_h,a_h) + r_h(s_h,a_h)^2).
	%	\end{align*}
	%	where $\|V_1^\pi\|^2_{d_1(\cdot)} = \sum_{s_1} d_1(s_1) V_1^\pi(s_1)^2$. 
\end{proof}
\subsection{Bounding the variance of $\widetilde{d}^\pi_h(s_h)$.}
Applying the definition of variance, we directly have
\begin{equation}\label{eq:reduction_to_variance}
\E\left[   \frac{\widetilde{d}_{h}^\pi(s_h)^2}{n_{s_h,a_h}}   \mathbf{1}(E_h)\right]  \leq  \frac{(1-\theta)^{-1}}{ n d_h^\mu(s_h,a_h)} \E\left[   \widetilde{d}_{h}^\pi(s_h)^2\right]  =   \frac{(1-\theta)^{-1}}{ n d_h^\mu(s_h,a_h)} (d_{h}^\pi(s_h)^2  + \Var[\widetilde{d}_{h}^\pi(s_h)]),
\end{equation}
where we use the fact that $\widetilde{d}^\pi_h(s_h)$ is unbiased (which can be proved by induction through applying total law of expectations and the recursive relationship $\widetilde{d}_t^\pi =  \widetilde{P}_{t}^\pi   \widetilde{d}_{t-1}^\pi$).
Therefore the only thing left is to bound the  the variance of $\widetilde{d}^\pi_h(s_h)$. To tackle it, we consider bounding the covariance matrix of $\widetilde{d}^\pi_h(s_h)$. As we shall see in Lemma~\ref{lem:ficticious_cov}, fortunately, we are able to derive an identical result of Lemma~B.4 in \citet{xie2019towards} for our \TMIS, which helps greatly in bounding the the variance of $\widetilde{d}^\pi_h(s_h)$.

\begin{lemma}[Covariance of $\widetilde{d}_h^\pi$ with TMIS]\label{lem:ficticious_cov}
	\begin{align*}
	\Cov(\widetilde{d}^\pi_h)  \preceq&  \frac{(1-\theta)^{-1}}{n}  \sum_{t=1}^{h-1} \P^\pi_{h+1,t+1} \diag\left[\sum_{s_{t},a_t}\frac{d_{t}^\pi(s_{t})^2 + \Var(\widetilde{d}_{t}^\pi(s_{t}))}{d_{t}^\mu(s_{t})} \frac{\pi(a_{t}| s_{t})^2}{\mu(a_{h}| s_{t})}  \P_{t+1,t}(\cdot|s_{t},a_{t})\right]   \left[\P^\pi_{h+1,t+1}\right]^T \\
	&+   \frac{1}{n} \P_{h,1}^\pi \diag\left[  d^\pi_1\right] [\P_{h,1}^\pi ]^T.
	\end{align*}
	where $\P^\pi_{h,t}  =  \P^\pi_{h,h-1}\cdot \P^\pi_{h-1,h-2}\cdot ... \cdot\P^\pi_{t+1,t}$ --- the transition matrices under policy $\pi$ from time $t$ to $h$ (define $\P_{h,h}^\pi := I$).
\end{lemma}

\begin{proof}[Proof of Lemma~\ref{lem:ficticious_cov}]
	We start by applying the law of total variance to obtain the following recursive equation
	\begin{align}
	\Cov[\widetilde{d}^\pi_h] &=  \E\left[ \Cov\left[  \widetilde{\P}_{h,h-1}^\pi \widetilde{d}^\pi_{h-1}   \middle|  \mathcal{D}_{h-1}\right]  \right]   + \Cov\left[ \E\left[\widetilde{\P}_{h,h-1}^\pi \widetilde{d}^\pi_{h-1} \middle|  \mathcal{D}_{h-1} \right] \right] \\
	&=\E\left[ \Cov\left[  \sum_{s_{h-1}}\widetilde{\P}_{h,h-1}^\pi(\cdot |s_{h-1}) \widetilde{d}^\pi_{h-1}(s_{h-1})   \middle|  \mathcal{D}_{h-1}\right]  \right]   + \Cov\left[ \E\left[\widetilde{\P}_{h,h-1}^\pi \widetilde{d}^\pi_{h-1} \middle|  \mathcal{D}_{h-1} \right] \right] \\
	&=  \underbrace{\E\left[ \sum_{s_{h-1}}\Cov\left[  \widetilde{\P}_{h,h-1}^\pi(\cdot | s_{h-1})\middle|  \mathcal{D}_{h-1}\right] \widetilde{d}^\pi_{h-1}(s_{h-1})^2 \right]}_{(*)}     +  \P_{h,h-1}^\pi \Cov[\widetilde{d}^\pi_{h-1}] [\P_{h,h-1}^\pi]^T. \label{eq:decomp_fict_marginal}
	\end{align}
	The decomposition of the covariance in the third line uses that $\Cov(X+Y) = \Cov(X) + \Cov(Y)$ when $X$ and $Y$ are statistically independent and the columns of $\widetilde{\P}_{h,h-1}$ are independent when conditioning on $\mathcal{D}_{h-1}$.

	\begin{align}
	(*)=& \E\left[\sum_{s_{h-1}}\sum_{a_{h-1}} \pi(a_{h-1}|s_{h-1})^2  \Cov\left[  \widetilde{\P}_h(\cdot|s_{h-1},a_{h-1}) \middle|  \text{Data}_{h-1} \right]  \widetilde{d}_{h-1}^\pi(s_{h-1})^2\right] \\
	=&\E\left[\sum_{s_{h-1}}\sum_{a_{h-1}} \pi(a_{h-1}|s_{h-1})^2 \mathbf{1}(E_{h-1}) \Cov\left[  \widehat{\P}_h(\cdot|s_{h-1},a_{h-1}) \middle|  \text{Data}_{h-1} \right]  \widetilde{d}_{h-1}^\pi(s_{h-1})^2\right] \\
	=& \E\left[\sum_{s_{h-1}}\sum_{a_{h-1}} \pi(a_{h-1}|s_{h-1})^2 \frac{\mathbf{1}(E_{h-1})}{n_{s_{h-1},a_{h-1}}} \Cov\left[ \mathbf{e}_{s^{(1)}_h} \middle|  s^{(1)}_{h-1}=s_{h-1},a^{(1)}_{h-1}=a_{h-1} \right]  \widetilde{d}_{h-1}^\pi(s_{h-1})^2\right] \\
	=&\sum_{s_{h-1},a_{h-1}}\pi(a_{h-1}|s_{h-1})^2\E\left[\frac{\widetilde{d}^\pi_{h-1}(s_{h-1})^2}{n_{s_{h-1},a_{h-1}}}\mathbf{1}(E_{h-1})\right]\bigg[\diag[ \P_{h}(\cdot|s_{h-1},a_{h-1})]\\
	&-\P_{h}(\cdot|s_{h-1},a_{h-1})\cdot \P_{h}(\cdot|s_{h-1},a_{h-1})^T\bigg]\\
	\prec& \sum_{s_{h-1}}\sum_{a_{h-1}} \Big\{  \frac{ d^\pi_{h-1}(s_{h-1})^2 + \Var[\widetilde{d}^\pi_{h-1}(s_{h-1})]}{n d^\mu_{h-1}(s_{h-1}) (1-\theta)}  \frac{\pi(a_{h-1}| s_{h-1})^2}{\mu(a_{h-1}| s_{h-1})}  \diag[ \P_{h,h-1}(\cdot|s_{h-1},a_{h-1})] \Big\}  \label{eq:dominate_by_diagonal}
	\end{align}
	The second line uses the fact that conditional on $E_{h-1}^c$, the variance of $\widetilde{\P}(\cdot|s_{h-1},a_{h-1})$ is zero given $\text{Data}_h$. The third line uses the basic property of empirical average, and the fourth line comes from the fact
	\begin{align*}
	&\Cov\left[  \mathbf{e}_{s_{h}^{(1)}}   \middle|  s_{h-1}^{(1)}=s_{h-1},a_{h-1}^{(1)}=a_{h-1} \right]\\
	=&\E\left[\mathbf{e}_{s_{h}^{(1)}}\cdot \mathbf{e}_{s_{h}^{(1)}}^T\middle|  s_{h-1}^{(1)}=s_{h-1},a_{h-1}^{(1)}=a_{h-1} \right]\\
	&-\E\left[\mathbf{e}_{s_{h}^{(1)}}\middle|  s_{h-1}^{(1)}=s_{h-1},a_{h-1}^{(1)}=a_{h-1} \right]\cdot \E\left[\mathbf{e}_{s_{h}^{(1)}}\middle|  s_{h-1}^{(1)}=s_{h-1},a_{h-1}^{(1)}=a_{h-1} \right]^T\\
	=& \diag(\P_{h,h-1}(\cdot| s_{h-1},a_{h-1}))  -  \P_{h,h-1}(\cdot|s_{h-1},a_{h-1}) [\P_{h,h-1}(\cdot|s_{h-1},a_{h-1})]^T
	\end{align*}
	
	The last line \eqref{eq:dominate_by_diagonal} uses the fact that $\P_{h,h-1}^\pi(\cdot|s_{h-1}) [\P_{h,h-1}^\pi(\cdot|s_{h-1})]^T$ is positive semidefinite, $n_{s_{h-1},a_{h-1}}\geq nd^\mu_{h-1}(s_{h-1},a_{h-1})(1-\theta)$ and the definition of variance for $\widetilde{d}^\pi_{h-1}(s_{h-1})$. Combining \eqref{eq:decomp_fict_marginal} and \eqref{eq:dominate_by_diagonal} and by recursively apply them, we get the stated results.
\end{proof}

Benefitting from the identical semidefinite ordering bound on $\text{Cov}(\widetilde{d}^\pi_h)$ for TMIS and SMIS, we can borrow the following results from \citet{xie2019towards} for our \TMIS.

\begin{lemma}[Corollary~2 of \citet{xie2019towards}]\label{lem:variance_bound}
	For $h=1$, we have
	$\Var[\widetilde{d}_1^\pi(s_1)] = \frac{1}{n}(d_h^\pi(s_1) -  d_h^\pi(s_1)^2)$, and for $h= 2,3,...,H$, we have:
	$$
	\Var[\widetilde{d}_h^\pi(s_h)]  \leq \frac{(1-\theta)^{-1}}{n} \sum_{t=2}^h \sum_{s_t}  \P_{h,t}^{\pi}( s_h | s_t )^2 \varrho(s_t)   + \frac{1}{n}\sum_{s_1} \P_{h,1}^{\pi}( s_h | s_1)^2 d_{1}(s_1)
	$$
	where 
	$\varrho(s_t) := \sum_{s_{t-1}} \left(\frac{d_{t-1}^\pi(s_{t-1})^2 + \Var(\widetilde{d}_{t-1}^\pi(s_{t-1}))}{d_{t-1}^\mu(s_{t-1})} \sum_{a_{t-1}}\frac{\pi(a_{t-1}|s_{t-1})^2}{\mu(a_{t-1}|s_{t-1})}   \P_{t,t-1}(s_t|s_{t-1},a_{t-1})\right).$
\end{lemma}

\begin{lemma}[Error propagation: Theorem~B.1 of \citet{xie2019towards}]\label{lem:fictitious_error_prop}
	Let $\tau_a := \max_{t,s_t,a_t}\frac{\pi(a_t|s_t)}{\mu(a_t|s_t)}$ and $\tau_s := \max_{t,s_t} \frac{d_t^\pi(s_t)}{d_t^\mu(s_t)}$. If $n \geq \frac{2(1-\theta)^{-1}t \tau_a \tau_s}{\max\{d_{t}^\pi(s_{t}),d_{t}^\mu(s_{t})\} }$ for all $t=2,...,H$, then for all $h=1,2,...,H$ and $s_h$, we have that:
	$$\Var[\widetilde{d}_h^\pi(s_h)]  \leq \frac{2(1-\theta)^{-1}h \tau_a\tau_s}{n} d_{h}^\pi(s_{h}).$$
\end{lemma}

Before giving the proof of Theorem~\ref{thm:main}, we first prove Lemma~\ref{lem:H3_to_H2}.

\begin{proof}[Proof of Lemma~\ref{lem:H3_to_H2}]
	Let value function $V^\pi_h(s_h)=\E_\pi[\sum_{t=h}^H r_t^{(1)}|s_h^{(1)}=s_h]$ and $Q$-function $Q^\pi_h(s_h,a_h)=\E_\pi[\sum_{t=h}^H r_t^{(1)}|s_h^{(1)}=s_h,a_h^{(1)}=a_h]$, then by total law of variance we obtain (let's suppress the policy $\pi$ for simplicity): 
	\begin{equation}\label{eq:H_to_ineq}
	\begin{aligned}
	&\mathrm{Var}\left[ \sum_{t=1}^h r_t^{(1)}+V_{h+1}(s_{h+1}^{(1)}) \right] \\
	=&\E\left[ \mathrm{Var}\bigg[\sum_{t=1}^h r_t^{(1)}+V_{h+1}(s_{h+1}^{(1)})\bigg| \mathcal{D}_h\bigg] \right] +\mathrm{Var}\left[ \E\bigg[\sum_{t=1}^h r_t^{(1)}+V_{h+1}(s_{h+1}^{(1)})\bigg| \mathcal{D}_h\bigg] \right]\\
	=&\E\left[ \mathrm{Var}\bigg[ r_h^{(1)}+V_{h+1}(s_{h+1}^{(1)})\bigg| s_h^{(1)},a_h^{(1)}\bigg] \right]
	+\mathrm{Var}\left[ \sum_{t=1}^{h-1} r_t^{(1)}+\E\bigg[V_{h+1}(s_{h+1}^{(1)})+r^{(1)}_h\bigg| s_h^{(1)},a_h^{(1)}\bigg] \right]\\
	=&\E\left[ \mathrm{Var}\bigg[ r_h^{(1)}+V_{h+1}(s_{h+1}^{(1)})\bigg| s_h^{(1)},a_h^{(1)}\bigg] \right]+\mathrm{Var}\left[ \sum_{t=1}^{h-1} r_t^{(1)}+Q_h(s_h^{(1)},a_h^{(1)}) \right]\\
	=&\E\left[ \mathrm{Var}\bigg[ r_h^{(1)}+V_{h+1}(s_{h+1}^{(1)})\bigg| s_h^{(1)},a_h^{(1)}\bigg] \right]+\E\left[ \mathrm{Var}\bigg[ \sum_{t=1}^{h-1} r_t^{(1)}+Q_h(s_h^{(1)},a_h^{(1)}) \middle|s^{(1)}_h,r^{(1)}_{1:h-1}\bigg]\right] \\
	+&\mathrm{Var}\left[ \E\bigg[ \sum_{t=1}^{h-1} r_t^{(1)}+Q_h(s_h^{(1)},a_h^{(1)}) \middle|s^{(1)}_h,r^{(1)}_{1:h-1}\bigg]\right]\\
	=& \E\left[ \mathrm{Var}\bigg[ r_h^{(1)}+V_{h+1}(s_{h+1}^{(1)})\bigg| s_h^{(1)},a_h^{(1)}\bigg] \right]+\E\left[ \mathrm{Var}\bigg[ Q_h(s_h^{(1)},a_h^{(1)}) \middle|s^{(1)}_h,r^{(1)}_{1:h-1}\bigg]\right]\\
	+&\mathrm{Var}\left[  \sum_{t=1}^{h-1} r_t^{(1)}+\E\bigg[Q_h(s_h^{(1)},a_h^{(1)}) \middle|s^{(1)}_h\bigg]\right]\\
	=& \E\left[ \mathrm{Var}\bigg[ r_h^{(1)}+V_{h+1}(s_{h+1}^{(1)})\bigg| s_h^{(1)},a_h^{(1)}\bigg] \right]+\E\left[ \mathrm{Var}\bigg[ Q_h(s_h^{(1)},a_h^{(1)}) \middle|s^{(1)}_h\bigg]\right]+\mathrm{Var}\left[  \sum_{t=1}^{h-1} r_t^{(1)}+V_h(s_h^{(1)})\right],\\
	\end{aligned}
	\end{equation}
	where we use Markovian property that $(V_{h+1}(s_{h+1}^{(1)})| \mathcal{D}_h)$ equals $(V_{h+1}(s_{h+1}^{(1)})| s_h^{(1)},a_h^{(1)})$ in distribution and $\E\bigg[V_{h+1}(s_{h+1}^{(1)})+r^{(1)}_h\bigg| s_h^{(1)},a_h^{(1)}\bigg]=Q_h(s_h^{(1)},a_h^{(1)})$. Then by applying \eqref{eq:H_to_ineq} recursively and letting $h=H$, we get the stated result. 
\end{proof}

\begin{remark}
	A straight forward implication of Lemma~\ref{lem:H3_to_H2} is the following:
	\[
	\sum_{t=1}^{H}\E_\pi\left[ \Var\Big[V_{t+1}^\pi(s_{t+1}^{(1)})+r_t^{(1)}\Big| s_{t}^{(1)}, a_t^{(1)}\Big]\right] \leq H^2R_{\max }^2.
	\]
\end{remark}

Combing Lemma~\ref{lem:var_decomp_fictitious} and \ref{lem:fictitious_error_prop}, we are now ready to prove the main Theorem~\ref{thm:main}.

\begin{proof}[Proof of Theorem~\ref{thm:main}]
	Plug the result of Lemma~\ref{lem:fictitious_error_prop} into Lemma~\ref{lem:var_decomp_fictitious} and uses the unbiasedness of $\widetilde{v}^\pi_{\text{TMIS}}$ (Lemma~\ref{lem:unbiasedness_fictitious}) we obtain $\forall \;0<\theta<1$:
	%
	%When $0\leq x\leq 1/2, (1-x)^{-1}\leq(1+2x)$
	\begin{equation}\label{eq:thm_main_proof_second_term}
	\begin{aligned}
	&\E[ ( \widetilde{v}_{\mathrm{TMIS}}^\pi  - v^\pi)^2] \\
	\leq&\frac{\Var[V_{1}^\pi(s_1^{(1)})]}{n} 
	+  \sum_{h=1}^H \sum_{s_h,a_h}  \frac{(1-\theta)^{-1}}{ n d_h^\mu(s_h,a_h)} d_{h}^\pi(s_h)^2  \pi(a_h|s_h)^2  \Var\left[ (V_{h+1}^\pi(s_{h+1}^{(1)}) +  r_h^{(1)})\middle| s_{h}^{(1)}=s_h,a_h^{(1)}=a_h\right].\\
	&+\frac{(1-\theta)^{-1}}{n}\sum_{h=1}^H \sum_{s_h,a_h}   \frac{2(1-\theta)^{-1}h\tau_a\tau_s}{n} \frac{d_{h}^\pi(s_h)}{d_{h}^\mu(s_h)} \frac{\pi(a_h|s_h)^2}{\mu(a_h|s_h)} \Var\left[ (V_{h+1}^\pi(s_{h+1}^{(1)}) +  r_h^{(1)})\middle| s_{h}^{(1)}=s_h\right]
	%\leq&   \frac{\|V_1^\pi\|^2_{d_1(\cdot)}}{n} + \frac{4}{n}\sum_{h=1}^H \sum_{s_h}   \frac{d_{h}^\pi(s_h)^2}{d_{h}^\mu(s_h)}  \sum_{a_h}\frac{\pi(a_h|s_h)^2}{\mu(a_h|s_h)}  (\| V_{h+1}^{\pi}\|_{P_{h+1,h}(\cdot|s_h,a_h)}^2 +\sigma^2 + R_{\max}^2)\\
	%&+ \frac{2}{n}\sum_{h=1}^H \sum_{s_h}   \frac{4h\tau_a\tau_s}{n}  \frac{d_{h}^\pi(s_h)}{d_{h}^\mu(s_h)}  \sum_{a_h}\frac{\pi(a_h|s_h)^2}{\mu(a_h|s_h)}  (\| V_{h+1}^{\pi}\|_{P_{h+1,h}(\cdot|s_h,a_h)}^2 +\sigma^2 + R_{\max}^2)\\
	%&+3H^3SR_{\max}^2 \left(\frac{2}{e}\right)^{ \frac{n \min_{t,s_t}d^{\mu}_t(s_t)}{2}}.
	%&+ \frac{4}{n}\sum_{h=1}^H \sum_{s_h}   \frac{4h\tau_a\tau_s}{n}  \frac{d_{h}^\pi(s_h)}{d_{h}^\mu(s_h)}  \sum_{a_h}\frac{\pi(a_h|s_h)^2}{\mu(a_h|s_h)}  (\sigma^2 + H^2R_{\max}^2)\\
	%&+ \frac{4}{n}\sum_{h=1}^H \sum_{s_h}   \frac{1}{d_{h}^\mu(s_h)} \left\{ \frac{2}{n}\sum_{t=2}^h\sum_{s_t}  \P_{h,t}^{\pi}( s_h | s_t )^2 \sum_{s_{t-1}}\frac{d_{t-1}^\pi(s_{t-1})^2}{d_{t-1}^\mu(s_{t-1})}\sum_{a_{t-1}}\frac{\pi(a_{t-1}|s_{t-1})^2}{\mu(a_{t-1}|s_{t-1})} \P_{t,t-1}(s_{t}|s_{t-1},a_{t-1})\right. \\
	% &\left.+ \frac{\sum_{s_1}\P_{h,1}^\pi(s_h|s_1)^2d_1(s_1)}{n} \right\} \sum_{a_h}\frac{\pi(a_h|s_h)^2}{\mu(a_h|s_h)}  (\sigma^2 + H^2R_{\max}^2)\\
	% \leq & \frac{5H^3R_{\max}^2  \eta_s\eta_a}{n}+ \frac{8}{n^2}\sum_{h=1}^H \sum_{t=1}^H \sum_{s_h}\frac{1}{d_{h}^\mu(s_h)}  \sum_{s_t} \P_{h,t}^\pi(s_h|s_t)^2d_t^\pi(s_t)  \eta_s \eta_a  \sum_{a_h}\frac{\pi(a_h|s_h)^2}{\mu(a_h|s_h)}  (\sigma^2 + H^2R_{\max}^2)	
	\end{aligned}
	\end{equation}
	Choose $\theta= \sqrt{4\log(n)/(n\min_{t,s_t,a_t}d_t^\mu(s_t,a_t))}$.
	Then by assumption $n > \frac{16\log n}{\min_{t,s_t,a_t}d_t^\mu(s_t,a_t)}$ we have $\theta<1/2$, which allows us to write $(1-\theta)^{-1}\leq (1+2\theta)$ in the leading term and $(1-\theta)^{-1}\leq 2$ in the subsequent terms. 
	The condition of Lemma~\ref{lem:fictitious_error_prop} is satisfied by The second assumption on $n$. Then, combining \eqref{eq:thm_main_proof_second_term} with Lemma~\ref{lem:fictitious_approximation} we get:
	\begin{equation}\label{eq:last_mian_thm}
	\begin{aligned}
	&\E[ (\widehat{v}_{\mathrm{TMIS}}^\pi -  v^\pi)^2]
	\leq  \frac{1}{n}\sum_{h=0}^H \sum_{s_h,a_h}  \frac{ d_{h}^\pi(s_h)^2}{d_{h}^\mu(s_h)}\frac{\pi(a_h|s_h)^2}{\mu(a_h|s_h)} \Var\left[ (V_{h+1}^\pi(s_{h+1}^{(1)}) +  r_h^{(1)})\middle| s_{h}^{(1)}=s_h,a_h^{(1)}=a_h\right] \\
	\cdot& \left(1+\sqrt{\frac{16\log n}{n\min_{t,s_t}d_t^\mu(s_t)}}\right)+\frac{3}{n^2}H^3SAR_{\max}^2\\
	+&\frac{8\tau_a\tau_s}{n^2}\sum_{h=1}^H\sum_{s_h,a_h} \frac{h\cdot d_{h}^\pi(s_h)}{d_{h}^\mu(s_h)}\frac{\pi(a_h|s_h)^2}{\mu(a_h|s_h)}
	\cdot\Var\left[ (V_{h+1}^\pi(s_{h+1}^{(1)}) +  r_h^{(1)})\middle| s_{h}^{(1)}=s_h,a_h^{(1)}=a_h\right], 	
	\end{aligned}
	\end{equation}
	now use Lemma~\ref{lem:H3_to_H2}, we can bound the last term in \eqref{eq:last_mian_thm} by 
	\[
	\frac{8\tau_a^2\tau_s H}{n^2\cdot d_m}\sum_{t=1}^{H}\E_\pi\left[ \Var\Big[V_{t+1}^\pi(s_{t+1}^{(1)})+r_t^{(1)}\Big| s_{t}^{(1)}, a_t^{(1)}\Big]\right]\leq 	\frac{8\tau_a^2\tau_s H^3R_{\max}^2}{n^2\cdot d_m},
	\]
	Combine this term with $\frac{3}{n^2}H^3SAR_{\max}^2$ we obtain the higher order term $	O(\frac{\tau_a^2\tau_s H^3R_{\max}^2}{n^2\cdot d_m})$, where we use that pigeonhole principle implies that $S<\tau_s,A<\tau_a$.

	This completes the proof.

\end{proof}

\section{Proofs of data splitting \TMIS.}
We define the fictitious data splitting \TMIS{} as:
\[
\widetilde{v}^\pi_{\text{split}}=\frac{1}{N}\sum_{i=1}^N \widetilde{v}^\pi_{(i)},
\]
where each $\widetilde{v}^\pi_{(i)}$ is the fictitious \TMIS{} of $\widehat{v}^\pi_{(i)}$. Moreover, we set all $\widetilde{v}^\pi_{(1)},\widetilde{v}^\pi_{(2)},...,\widetilde{v}^\pi_{(N)}$ jointly share the same fictitious parameter $\theta_M$.

\begin{proof}[Proof of Theorem~\ref{thm:var_bound_data_splitting_TMIS }]
	Let $E':=\{\exists\; \widetilde{v}^\pi_{(i)}:s.t. \widetilde{v}^\pi_{(i)}\neq \widehat{v}^\pi_{(i)}\}$, then an argument similar to Lemma~\ref{lem:fictitious_approximation} can be derived:
	\[
	\E[(\widehat{v}_{\mathrm{split}}^\pi  - v^\pi)^2] 
	\leq 3 \P[E']H^2R_{\max}^2+ \E[(\widetilde{v}^\pi_{\mathrm{split}} - v^\pi)^2], %= 3 \P[E]R_{\max}^2+ \Var[\widetilde{v}^\pi].
	\]
	and 
	\[
	\P[E']\leq N \sum_{t}\sum_{s_t}\sum_{a_t} \P[n_{s_t,a_t} < M\cdot d^{\mu}_t(s_t,a_t)(1-\theta_M)] \leq NHSA e^{-\frac{\theta_M^2 M \min_{t,s_t,a_t} d^{\mu}_t(s_t,a_t)}{2}},
	\]
	therefore $\P[E']$ will be sufficiently small if $M\geq O(\mathrm{Polylog}(H,S,A,n)/\min_{t,s_t,a_t}d^\mu_t(s_t,a_t))$. By near-uniformity we  $M \geq O(\mathrm{Polylog}(H,S,A,n)SA) \geq O(\mathrm{Polylog}(H,S,A,n)/\min_{t,s_t,a_t}d^\mu_t(s_t,a_t))$.
	
	Moreover, by i.i.d and unbiasedness of $\widetilde{v}^\pi_{(i)}$, we have
	\[
	\E[(\widetilde{v}^\pi_{\mathrm{split}} - v^\pi)^2]=\frac{1}{N}\E[(\widetilde{v}^\pi_{(1)} - v^\pi)^2]\leq \frac{1}{N}\cdot O(\frac{H^2SA}{M})=O(\frac{H^2SA}{n}),
	\]
	by the second assumption on $M$ and Theorem~\ref{thm:main}.
	
\end{proof}

We now proof Lemma~\ref{lem:uniform_bound}, since it will be used to as the intermediate step for proving Theorem~\ref{thm:high_probability}. 
\begin{proof}[Proof of Lemma~\ref{lem:uniform_bound}]
	Note that 
	\begin{align*}
	&\P\left[ \left\{ \exists \pi\in\prod \;s.t.\;\widetilde{v}^\pi_{\mathrm{split}}\neq \widehat{v}^\pi_{\mathrm{split}}\right\} \right] \leq N\cdot \P\left[ \left\{ \exists \pi\in\prod, \;s.t.\;\widetilde{v}^\pi_{{(1)}}\neq \widehat{v}^\pi_{{(1)}}\right\} \right] \\
	&\leq N\cdot \P\left[ \{\exists t, s_t,a_t \;\text{s.t.}\; n_{s_t,a_t}^{(1)} <  n d_t^\mu(s_t,a_t)(1-\theta_M)  \}\right] \\
	&\leq NHSA e^{-\frac{\theta_M^2 M \min_{t,s_t,a_t} d^{\mu}_t(s_t,a_t)}{2}},
	\end{align*}
	therefore by near-uniformity $M > \max\left[ O(SA\cdot \mathrm{Polylog} (S,H,A,N,1/\delta)),O(H \tau_a \tau_s)  \right]$ is sufficient to guarantee the stated result.
\end{proof}

Now we can prove Theorem~\ref{thm:high_probability}.

\begin{proof}[Proof of Theorem~\ref{thm:high_probability}]
	First of all, we have 
	\begin{equation}\label{eqn:decomp_error}
	\P\left( |\widehat{v}^\pi_{\mathrm{split}}-v^\pi|>\epsilon \right)\leq  \P\left( |\widehat{v}^\pi_{\mathrm{split}}-\widetilde{v}^\pi_{\mathrm{split}}|>0 \right)+\P\left( |\widetilde{v}^\pi_{\mathrm{split}}-{v}^\pi|>{\epsilon} \right),
	\end{equation}
	Now by Bernstein inequality we have 
	\begin{equation}\label{eqn:split_bernstein}
	\P\left( |\widetilde{v}^\pi_{\mathrm{split}}-{v}^\pi|>{\epsilon} \right)=\P\left( |\frac{1}{N}\sum_{i=1}^N(\widetilde{v}^\pi_{(i)}-v^\pi)|\geq \epsilon \right)\\
	\leq \exp{\left( -\frac{N\epsilon^2}{2\mathrm{Var}(\widetilde{v}^\pi_{(1)})+2HR_{\max}\epsilon/3}\right) }:=\delta/2.
	\end{equation}
	Solving \eqref{eqn:split_bernstein} and apply Theorem~\ref{thm:main}, we obtain 
	\begin{equation}\label{eqn:split_eps}
	%\sqrt{\frac{2\mathrm{Var}(\widetilde{v}^\pi_{(1)})\log(2/\delta)}{N}}+\frac{HR_{\text{max}}\log(2/\delta)}{3N}\leq
	\epsilon \leq \sqrt{\frac{2\mathrm{Var}(\widetilde{v}^\pi_{(1)})\log(2/\delta)}{N}}+\frac{2HR_{\text{max}}\log(2/\delta)}{3N}\leq \widetilde{O}(\sqrt{\frac{H^2SA\log(2/\delta)}{M\cdot N}})+\frac{2HR_{\text{max}}\log(2/\delta)}{3N} .
	\end{equation}
	As $N$ goes large, the square root term in \eqref{eqn:split_eps} will dominate and it seems we only need to consider the square root term in $N$ and treat the second term as the higher order term. However, since $M > \max\left[ O(SA\cdot \mathrm{Polylog} (S,H,A,N,1/\delta)),O(H \tau_a \tau_s)  \right]$, $N$ cannot be arbitrary large given $n$. An example is: when $N=n$, then $M=n/N=1$ does not satisfy the condition. Therefore to make the square root term dominates we need 
	\[
	\sqrt{\frac{H^2SA\log(2/\delta)}{M\cdot N}}\geq O(\frac{HR_{\text{max}}\log(2/\delta)}{N}).
	\] 
	This translates to 
	\begin{equation}\label{eqn:condition_for_M}
	M\leq \widetilde{O}(\sqrt{nSA}),
	\end{equation}
	where $\widetilde{O}$ absorbs all the \text{Polylog} terms. 
	
	Therefore under the condition \eqref{eqn:condition_for_M}, we can really absorb the second term in \eqref{eqn:split_eps} (as higher order term) and combine it with Lemma~\ref{lem:uniform_bound} to get that with probability $1-\delta$, 
	
	\[
	|\widehat{v}^\pi_{\mathrm{split}}-v^\pi|\leq 0+\widetilde{O}(\sqrt{\frac{H^2SA}{M\cdot N}})=\widetilde{O}(\sqrt{\frac{H^2SA}{n}}).
	\]

	%Then apply Theorem~\ref{thm:main} to $\mathrm{Var}(\widetilde{v}^\pi_{(1)})$ and exact out the higher order terms we obtain the result.  Note that the $n^{-0.75}$ terms come from the additive factor.
\end{proof}

\begin{proof}[Proof of Theorem~\ref{thm:unifrom_convergence}]
	The non-uniform result of Theorem~\ref{thm:high_probability} gives:
	\[
	|\widehat{v}_{\mathrm{split}}^\pi -  v^\pi|\leq \widetilde{O}(\sqrt{\frac{H^2SA}{n}})
	\]
	Note that all nonstationary deterministic polices class have cardinality $|\prod|=A^{HS}$, which implies $\log |\prod|=HS\log A$, therefore combine Lemma~\ref{lem:uniform_bound} with a direct union bound and Multiplicative Chernoff bound we obtain 
	\[
	\sup_{\pi\in\prod}	|\widehat{v}_{\mathrm{split}}^\pi -  v^\pi|\leq \widetilde{O}(\sqrt{\frac{H^3S^2A}{n}})
	\]
\end{proof}

% \section{Implication in stationary MDP setting.}
% \label{sec:implication}
% \red{We can derive a similar result of Theorem~\ref{thm:main} for stationary MDP and save a factor of $H$.} In fact, in the stationary setting we can set 
% \[
% \widehat{P}(s^\prime|s,a)=\frac{1}{n_{s,a}}\sum_{i,t}\mathbf{1}(s^{(i)}_{t+1}=s^\prime,s^{(i)}_t=s,a^{(i)}_t=a)
% \]
% where $n_{s,a}$ is total number of visitation to state-action pair $(s,a)$ among all $H$ steps (in contrast to $n_{s_t,a_t}$ in the non-stationary case, where only the number of visitation at time $t$ is considered). Moreover, the fictitious counterpart is defined as:
% \[
% \widetilde{P}(s^\prime|s,a)=\begin{cases}\widehat{P}(s^\prime|s,a) \quad&n_{s,a}>\Omega(n\sum_{t=1}^Hd^\mu_t(s,a))\\
% {P}(s^\prime|s,a)\quad&o.w.\end{cases}
% \] 

% In the case where $d^\mu_t(s,a)>\Omega({1/SA})$, then $n_{s,a}>\Omega(n\sum_{t=1}^Hd^\mu_t(s,a))$ implies $1/n_{s,a}<O(SA/(nH))$.  

\section{More details about Empirical Results.}
\label{sec:add_empirical}
\noindent\textbf{Restate Time-varying, non-mixing Tabular MDP in Section~\ref{sec:experiment}}. 

There are two states $s_0$ and $s_1$ and two actions $a_1$ and $a_2$. State $s_0$ always has probability $1$ going back to itself, regardless of the actions, \emph{i.e.} $P_t(s_0|s_0,a_1)=1$ and $P_t(s_0|s_0,a_2)=1$. For state $s_1$, at each time step there is one action (we call it $a$) that has probability $2/H$ going to $s_0$ and the other action (we call it $a'$) has probability $1$ going back to $s_1$, 
\[
P_t(s|s_1,a)=\begin{cases}
\frac{2}{H}\;\;&\text{if}\;s=s_0;\\
1-\frac{2}{H}\;\;&\text{if}\;s=s_1.\\
\end{cases}\quad P_t(s|s_1,a')=\begin{cases}
0\;\;&\text{if}\;s=s_0;\\
1\;\;&\text{if}\;s=s_1.\\
\end{cases}
\]
and which action will make state $s_1$ go to state $s_0$ with probability $2/H$ is decided by a random parameter $p_t$ uniform sampled in $[0,1]$. If $p_t<0.5$, $a=a_1$ and if $p_t\geq 0.5$, $a=a_2$. These $p_1,...,p_H$ are generated by a sequence of pseudo-random numbers.  Moreover, one can receive reward $1$ at each time step if $t>H/2$ and is in state $s_0$, and will receive reward $0$ otherwise. Lastly, for logging policy, we define it to be uniform:
\[
\mu(\cdot|s_0)=\begin{cases}\frac{1}{2}\;\;&\text{if} \;\cdot=a_1;\\
\frac{1}{2}\;\;&\text{if} \;\cdot=a_2.
\end{cases}
\quad\text{and}\quad \mu(\cdot|s_1)=\begin{cases}\frac{1}{2}\;\;&\text{if} \;\;\cdot=a_1;\\
\frac{1}{2}\;&\text{if} \;\cdot=a_2.
\end{cases}
\]
For target policy $\pi$, we define it as:

\[
\pi(\cdot|s_0)=\begin{cases}\frac{1}{2}\;\;&\text{if} \;\cdot=a_1;\\
\frac{1}{2}\;\;&\text{if} \;\cdot=a_2.
\end{cases}
\quad\text{and}\quad \pi(\cdot|s_1)=\begin{cases}\frac{1}{4}\;\;&\text{if} \;\;\cdot=a_1;\\
\frac{3}{4}\;&\text{if} \;\cdot=a_2.
\end{cases}
\]
We run this non-stationary MDP model in the {\fontfamily{cmtt}\selectfont Python} environment and pseudo-random numbers $p_t$'s are generated by keeping {\fontfamily{cmtt}\selectfont numpy.random.seed(100)}. 

We run each methods under $K = 100$ macro-replications with data $\mathcal{D}_{(k)}=\left\lbrace (s_t^{(i)},a_t^{(i)},r_t^{(i)})\right\rbrace^{i\in[n],t\in[H]}_{(k)} $, and use each $\mathcal{D}_{(k)}$ $(k=1,...,K)$ to construct a estimator $\widehat{v}^\pi_{[k]}$, then the (empirical) RMSE is computed as: 
\[
\mathrm{RMSE} = \sqrt{\frac{\sum_{k=1}^K (\widehat{v}^\pi_{[k]}-v^\pi_{\mathrm{true}})^2}{K}},
\]
where $v^\pi_{\mathrm{true}}$ is obtained by calculating  $P^\pi_{t+1,t}(s^\prime|s)=\sum_a P_{t+1,t}(s^\prime|s,a)\pi_t(a|s)$, the marginal state distribution ${d}_t^\pi  = {P}^{\pi}_{t,t-1} {d}_{t-1}^\pi$, ${r}^\pi_t(s_t)=\sum_{a_t}{r}_t(s_t,a_t)\pi_t(a_t|s_t)$ and ${v}^\pi_{\mathrm{true}}=\sum_{t=1}^H\sum_{s_t}{d}^\pi_t(s_t){r}^\pi_t(s_t)$. Then Relative-RMSE equals to $\mathrm{RMSE}/v^\pi_{\mathrm{true}}$.

\textbf{Other generic IS-based estimators.} There are other importance sampling based estimators including \emph{weighted importance sampling} (WIS) and \emph{importance sampling with stationary state distribution } (SSD-IS, \citet{liu2018breaking}). The empirical comparisons including these methods are well-demonstrated in \citet{xie2019towards} and it was empirically shown that they are worse than SMIS. Because of that, we only focus on comparing SMIS and TMIS in our simulation study.

% \begin{figure}
% 	\centering     %%% not \center
% 	\subfigure[Different Episode $n$]{\label{fig:different_n}\includegraphics[width=85mm]{TMIS1.pdf}}
% 	\subfigure[Different Horizon $H$]{\label{fig:different_H}\includegraphics[width=85mm]{TMIS.pdf}}
% 	\caption{Relative RMSE ($\sqrt{\text{MSE}}/v^\pi$) on ModelWin MDP}
% 	\label{fig:ModelWin}
% \end{figure}

%\leq& N\cdot \P\left[ \{\exists t, s_t,a_t \;\text{s.t.}\; n_{s_t,a_t} <  n d_t^\mu(s_t,a_t)(1-\theta)  \}\right] 
\begin{algorithm*}[thb]
	\caption{Data Splitting Tabular MIS OPE}
	\label{alg:secondalgo}
	{\bfseries Input:} Logging data $\mathcal D = \{\{s_t^{(i)},a_t^{(i)},r_t^{(i)}\}_{t = 1}^{H }\}_{i = 1}^{n}$ from the behavior policy $\mu$. A target policy $\pi$ which we want to evaluate its cumulative reward. Splitting data size $M$. 
	\begin{algorithmic}[1]
		\STATE Randomly splitting the data $\mathcal{D}$ evenly into $N$ folds, with each fold $|\mathcal{D}^{(i)}|=M$.
		\FOR{$i = 1,2,\dotsc,N$}
		\STATE Use Algorithm~\ref{alg:mainalgo} to estimate $\widehat{v}_{(i)}^\pi$ with data $\mathcal{D}^{(i)}$.
		\ENDFOR
		\STATE Use the mean of $\widehat{v}_{(1)}^\pi,\widehat{v}_{(2)}^\pi,...,\widehat{v}_{(N)}^\pi$ as the final estimation of $v^\pi$.
	\end{algorithmic}
\end{algorithm*}

\end{document}